\documentclass[conference,compsoc]{IEEEtran}
\IEEEoverridecommandlockouts
\usepackage{caption}
\usepackage{subcaption}
\usepackage{graphicx}
\usepackage{mathtools}
\usepackage{pifont}        
\usepackage{adjustbox}     
\usepackage{makecell}      
\usepackage{booktabs}      
\usepackage{tabularx}      
\usepackage{ragged2e}      

\usepackage{xcolor}

\usepackage{marvosym}
\usepackage{makecell}

\newcommand{\cmark}{\ding{51}}
\newcommand{\xmark}{\ding{55}}
\usepackage{amsmath}
\usepackage{amssymb}
\usepackage{amsthm}
\usepackage{bm} 
\usepackage{threeparttable}
\usepackage{algorithm}
\usepackage{algpseudocode}
\usepackage{algorithmicx}
\usepackage{pifont}
\usepackage{booktabs}
\usepackage{array}
\usepackage{makecell}
\usepackage{multirow}
\usepackage{cite}
\newcommand{\ie}{\textit{i.e\@.},\ }
\newcommand{\eg}{\textit{e.g\@.},\ }
\usepackage{hyperref}
\newtheorem{theorem}{Theorem}

\newcommand{\para}[1]{\vspace{2pt}\noindent{\textbf{#1}}\hspace{10pt}\vspace{0.1pt}}
\newenvironment{icompact}{
	\begin{list}{$\bullet$}{
			\itemindent -.05em
			\parsep 0pt plus 1pt
			\partopsep 0pt plus 1pt
			\topsep 2pt plus 2pt minus 2pt
			\itemsep 0pt plus 1.3pt
			\parskip 0pt plus 2pt
			\leftmargin 0.13in}
	}
	{\normalsize
	\end{list}
}

\ifCLASSINFOpdf
\else
\fi

\begin{document}\sloppy
	
\title{Toward Efficient Membership Inference Attacks against \\ Federated Large Language Models: A Projection Residual Approach}

\author{
	\IEEEauthorblockN{Guilin Deng\textsuperscript{*}, Silong Chen\textsuperscript{*}, Yuchuan Luo\textsuperscript{\Letter}, Yi Liu\textsuperscript{\Letter}\IEEEauthorrefmark{2}, Songlei Wang\IEEEauthorrefmark{3}, \\ Zhiping Cai, Lin Liu, Xiaohua Jia\IEEEauthorrefmark{2}, Shaojing Fu}
	\IEEEauthorblockA{Colleague of Computer Science and Technology, National University of Defense Technology, Changsha, China}
	\IEEEauthorblockA{\IEEEauthorrefmark{2} City University of Hong Kong, Hong Kong, China}
	\IEEEauthorblockA{\IEEEauthorrefmark{3} Shenzhen University, Shenzhen, China}
	\IEEEauthorblockA{\{dengguilin, chensilong, luoyuchuan09\}@nudt.edu.cn}
}

\maketitle
	
\renewcommand{\thefootnote}{*}
\makeatletter
\long\def\@makefntext#1{%
	\parindent 1em\noindent
	\hb@xt@1.8em{\hss\@thefnmark}%
	\ #1}
\makeatother
\footnotetext[1]{Equal contribution}
	
\begin{abstract}
		
Federated Large Language Models (FedLLMs) enable multiple parties to collaboratively fine-tune LLMs without sharing raw data, addressing challenges of limited resources and privacy concerns. Despite data localization, shared gradients can still expose sensitive information through membership inference attacks (MIAs). However, FedLLMs’ unique properties, \ie massive parameter scales, rapid convergence, and sparse, non-orthogonal gradients, render existing MIAs ineffective. To address this gap, we propose \texttt{ProjRes}, the first projection residuals-based passive MIA tailored for FedLLMs. \texttt{ProjRes} leverages hidden embedding vectors as sample representations and analyzes their projection residuals on the gradient subspace to uncover the intrinsic link between gradients and inputs. It requires no shadow models, auxiliary classifiers, or historical updates, ensuring efficiency and robustness. Experiments on four benchmarks and four LLMs show that \texttt{ProjRes} achieves near 100\% accuracy, outperforming prior methods by up to 75.75\%, and remains effective even under strong differential privacy defenses. Our findings reveal a previously overlooked privacy vulnerability in FedLLMs and call for a re-examination of their security assumptions. Our code and data are available at the available at \href{https://anonymous.4open.science/r/Passive-MIA-5268}{link}.
		
\end{abstract}

\IEEEpeerreviewmaketitle
	
\section{Introduction}
Large Language Models (LLMs) such as GPT \cite{GPT2}, BERT \cite{BERT}, Llama3 \cite{llama3}, and Qwen \cite{qwen2.5} have attracted substantial attention for their remarkable capabilities across a wide range of natural language understanding and generation tasks. Fine-tuning pre-trained LLMs on domain-specific datasets has become the standard approach for adapting these models to downstream applications~\cite{adapter_Heterogeneous}. However, the vast parameter scales of modern LLMs result in prohibitively high computational and time costs during fine-tuning \cite{adapter_efficent}, posing significant barriers for users with limited resources. This exacerbates the technological divide and hinders widespread adoption. In addition, the domain-specific corpora required for fine-tuning are often difficult to collect and are typically distributed across multiple institutions~\cite{federatedscope}. Due to stringent data privacy regulations such as the General Data Protection Regulation (GDPR)~\cite{GDPR}, these institutions are prohibited from directly sharing sensitive data, further complicating collaborative model development and deployment.
	
To address the two key challenges outlined above, Federated Large Language Models (FedLLMs)~\cite{openfedllm, FedLLM_future, FedLLM_FL2} provide a promising solution by enabling multiple participants integrate computational resources to collaboratively train models without directly sharing their local data. Instead, participants exchange model gradients or parameter updates, allowing for joint learning while preserving data privacy~\cite{federatedscope}. In a typical FedLLM setting, a central server distributes a pre-trained LLM to participating clients, who perform local fine-tuning using their private datasets and computational resources. The locally updated parameters are then sent back to the server for aggregation, thereby mitigating both data isolation and computational resource constraints. Moreover, the emergence of Parameter-Efficient Fine-Tuning (PEFT)~\cite{adapter_Dual-Personalizing, adapter, adapter_DP-DyLoRA} techniques has greatly enhanced the practicality of FedLLMs by substantially reducing communication and storage overhead during training. As a result, FedLLMs have gained significant traction in both academia and industry, with applications expanding rapidly across diverse domains such as biomedical analysis~\cite{biomedical1, biomedical2}, legal reasoning~\cite{legal1}, weather forecasting~\cite{weather1}, code generation~\cite{code1, code3}, and software engineering~\cite{code2}.

\begin{figure*}[t]
		\centering
		\begin{subfigure}[t]{0.23\textwidth}
			\includegraphics[width=\textwidth]{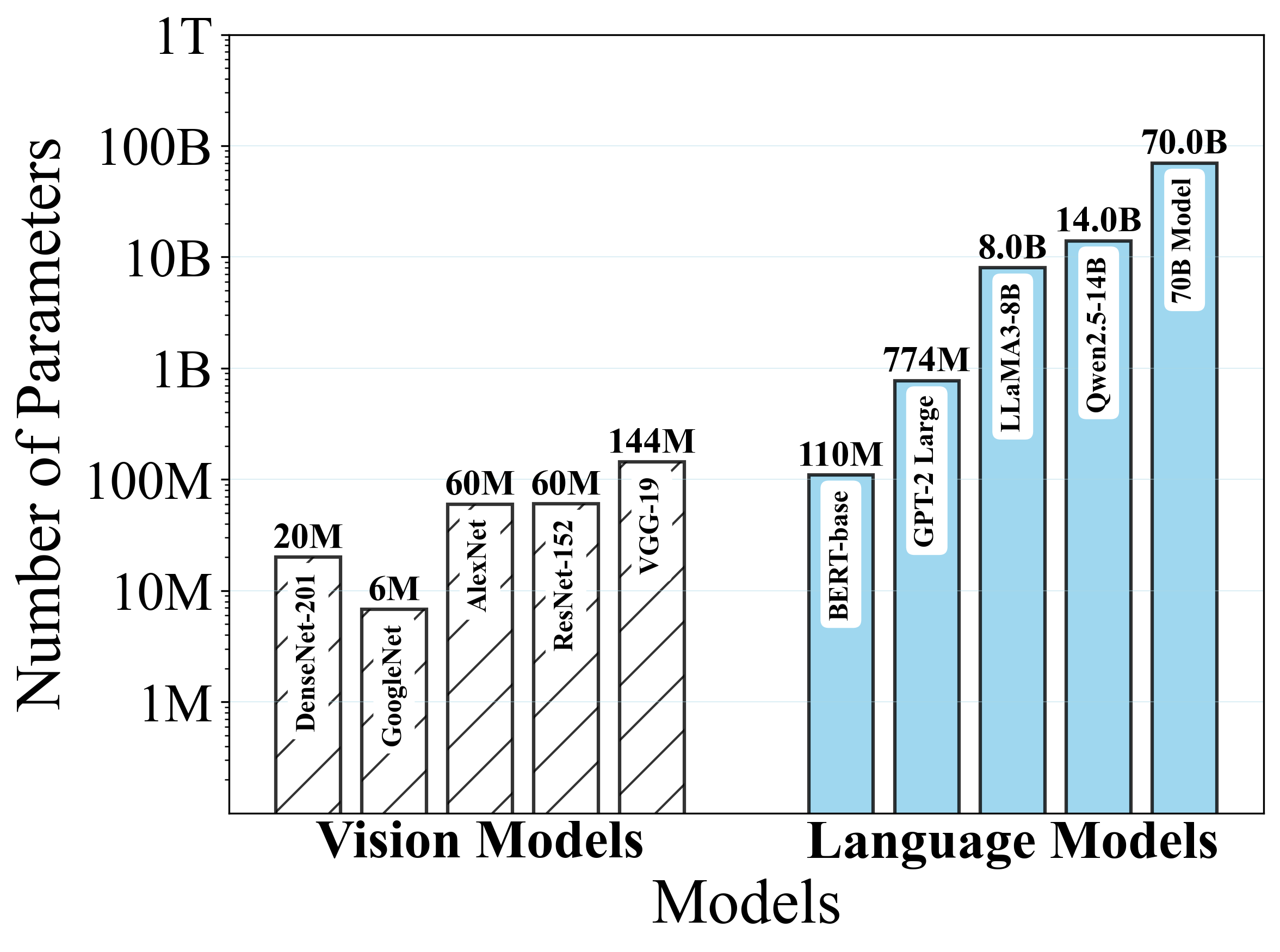}
			\caption{CV models vs. LLMs.}
			\label{fig:model_parameter}
		\end{subfigure}
		\hfill
		\begin{subfigure}[t]{0.23\textwidth}
			\includegraphics[width=\textwidth]{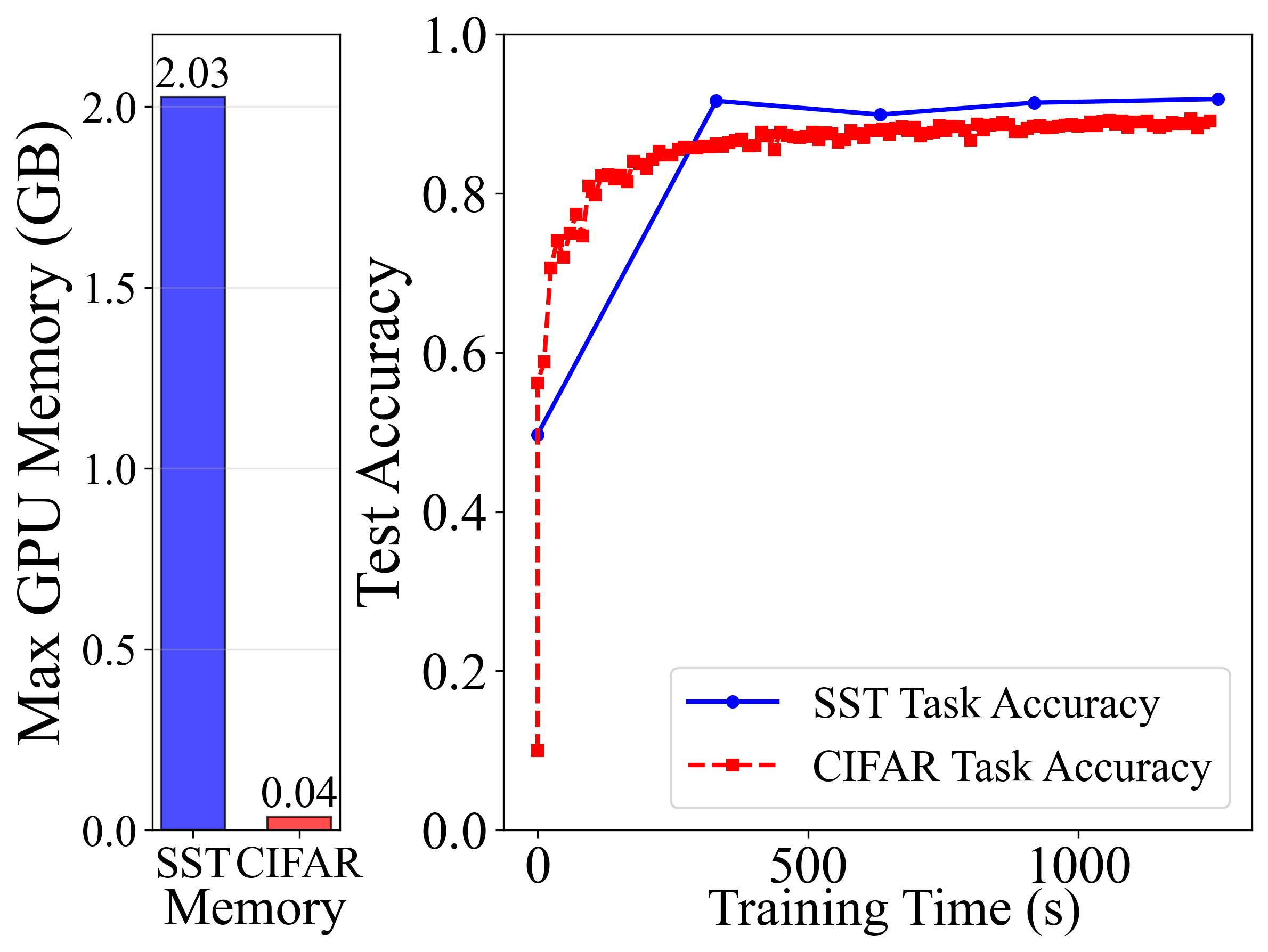} 
			\caption{Resource overhead.}
			\label{fig:convergence_time}
		\end{subfigure}
		\hfill
		\begin{subfigure}[t]{0.23\textwidth}
			\includegraphics[width=\textwidth]{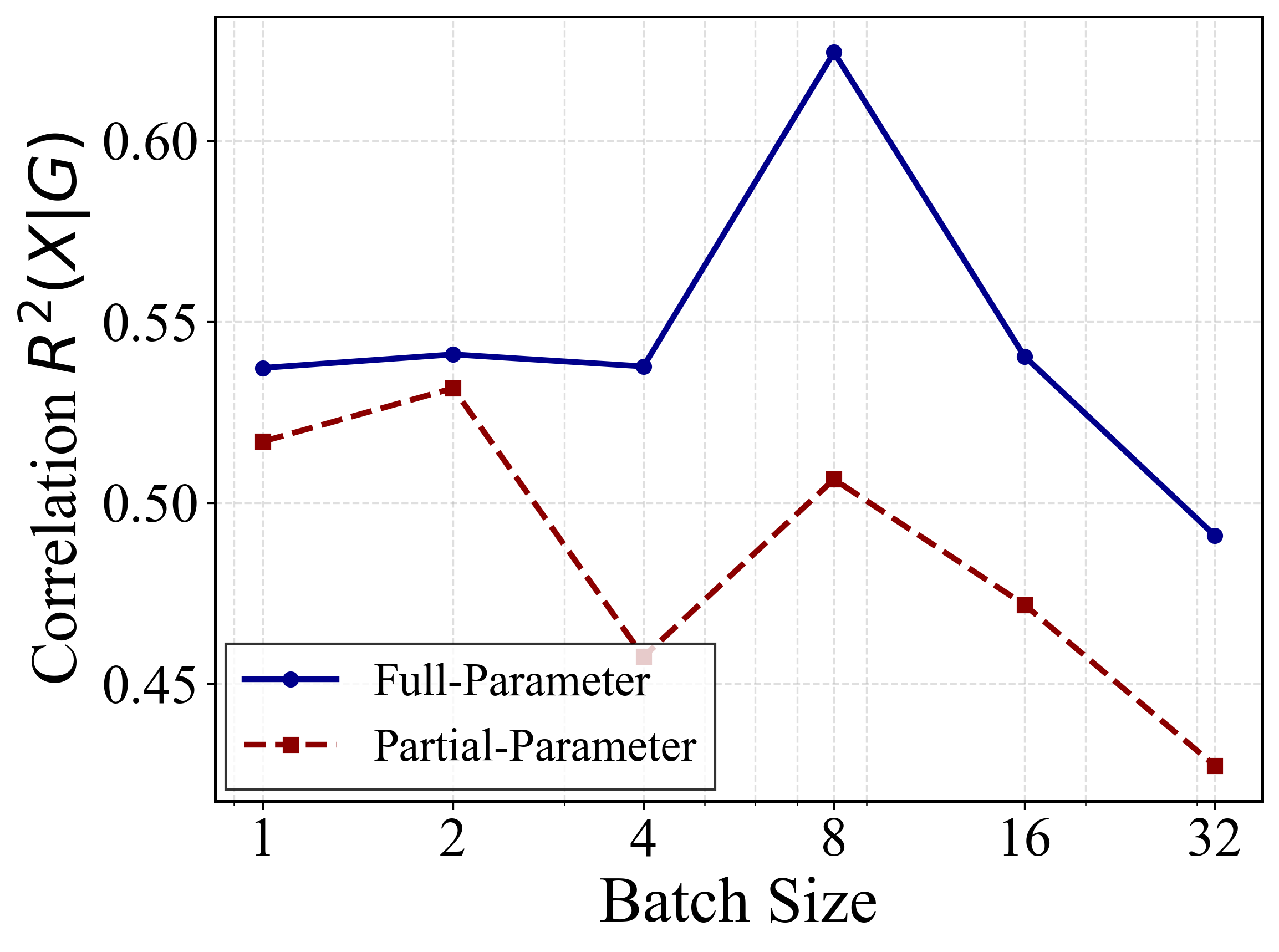}
			\caption{Gradient correlation.}
			\label{fig:grad_corr_batch}
		\end{subfigure}
		\hfill
		\begin{subfigure}[t]{0.23\textwidth}
			\includegraphics[width=\textwidth]{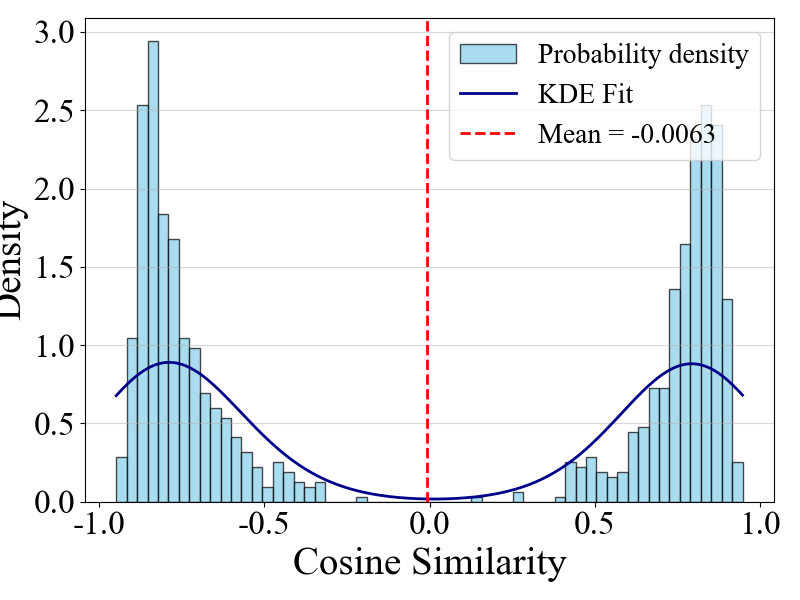}
			\caption{Gradient cosine similarity.}
			\label{fig:cosine_sim}
		\end{subfigure}
		\caption{Overview of the feasibility study results. (a) A glance at the complexity of traditional Computer Vision (CV) models and LLMs. (b) Comparison of shadow model convergence between language task (BERT-Base model~\cite{BERT} on SST-2~\cite{sst2} dataset) and vision task (CNN model on CIFAR-10~\cite{cifar10} dataset). (c) Comparison of gradient-gata correlation with BERT-Base model on SST-2 dataset. (d) Gradient cosine similarity comparison.}
		\label{fig:comparison_summary}
		\vspace{-0.3cm}
\end{figure*}
	
Although FedLLMs ensure that raw data remain local, their inheritance of the Federated Learning (FL) framework necessitates careful consideration of potential Membership Inference Attacks (MIAs)~\cite{learn-based_MIA1, learn-based_MIA2, liu2022membership, learn-based_MIA4} launched by honest-but-curious adversaries. Such adversaries may exploit proxy information, \eg gradients or parameter updates, to infer whether specific samples were part of the training data, thereby leaking sensitive information. While this threat has been extensively studied in traditional FL, its implications for FedLLMs remain underexplored. To bridge this gap, we conduct a feasibility study (see Fig. \ref{fig:comparison_summary}) to examine whether existing MIA techniques can effectively compromise the privacy of FedLLMs, as detailed below.
	
\begin{icompact}
		\item \textbf{Limited by the large scale of model parameters.} As shown in Fig. \ref{fig:model_parameter}, LLMs have orders of magnitude more parameters than traditional FL models, drastically increasing computational and memory costs for training shadow or attack models in existing MIAs \cite{learn-based_MIA1, learn-based_MIA2, liu2022membership, learn-based_MIA4}. Our experiments confirm this: training a shadow model on SST-2 \cite{sst2} requires over $50\times$ more GPU memory than on CIFAR-10 \cite{cifar10} (Fig. \ref{fig:convergence_time}). These excessive demands make current learning-based MIAs impractical for FedLLMs, underscoring the need for more efficient and scalable attack frameworks.
		
		\item \textbf{Limited by the model initialization strategy.} FedLLMs are typically initialized from pre-trained LLMs, enabling rapid convergence within a few epochs. This reduces repeated exposure of training samples during optimization, limiting adversaries' ability to capture the dynamic influence of individual data points on model updates \cite{FTA}. As shown in Fig. \ref{fig:convergence_time}, BERT-Base~\cite{BERT} model nearly converges after one epoch on SST-2, whereas that requires over twenty epochs on CIFAR-10.
		
		\item \textbf{Limited by fine-tuning and gradient semantics.} In FedLLMs, freezing a large portion of parameters, common in PEFT (\eg adapters, LoRA), significantly reduces the volume of client-uploaded gradients, thereby limiting the attack surface for gradient-based membership inference \cite{FedLLM_FL1, sparsified, sparsified2}. Moreover, as shown in Fig.~\ref{fig:grad_corr_batch}, gradients from multiple layers exhibit stronger correlations with input data than single-layer ones, implying deeper gradient exposure increases privacy risks. Notably, freezing the embedding layer prevents direct reconstruction of lexical content from its gradients \cite{embedding_MIA}. Additionally, unlike in traditional image tasks where gradients become approximately orthogonal upon convergence \cite{ICLR_MIA}, FedLLM gradients remain highly correlated across samples (Fig.~\ref{fig:cosine_sim}), obscuring individual contributions and complicating precise membership inference.
		
		
\end{icompact}
The above analysis indicates that existing MIAs are challenging to execute effectively due to the unique architectural and training characteristics of FedLLMs. Therefore, to further quantify and understand the privacy risk boundaries of FedLLMs, it is essential to investigate the actual extent of membership leakage induced by gradients during training.

In this paper, we propose \texttt{ProjRes}, a passive MIA based on projection residuals tailored for FedLLMs, which determines data membership by analyzing whether a sample’s hidden embedding lies within the linear subspace spanned by the client’s gradients. Specifically, \texttt{ProjRes} reformulates membership inference as the task of reconstructing the projection residuals of embeddings within the gradient space, thereby eliminating the need to train additional shadow or attack models, often LLMs with massive parameters, required by traditional MIAs. Because FedLLMs typically adopt initialization strategies that lead to rapid convergence, tracking the temporal evolution of member samples across multiple training rounds becomes impractical. To address this limitation, \texttt{ProjRes} exploits the intrinsic correlation between input representations and gradients, enabling efficient membership inference using only single-round gradient information. Furthermore, by reformulating membership inference as determining whether a sample’s embedding lies within the gradient span space, ProjRes effectively circumvents the challenge of performing membership inference without direct access to the gradient semantic space, which is tightly coupled with input features. Finally, \texttt{ProjRes} exploits the high sensitivity of projection residuals to subtle directional deviations, which amplifies the separation between member and non-member samples within the gradient subspace, thereby enhancing discriminative performance.

We conduct a comprehensive evaluation of \texttt{ProjRes} across four LLMs (\ie BERT-Base~\cite{BERT}, GPT-Large~\cite{GPT2}, Llama3-8B~\cite{llama3}, and Qwen2.5-14B~\cite{qwen2.5}) and four benchmark datasets (\ie CoLA-1.1 \cite{CoLA}, Yelp-tip \cite{yelp}, SST-5 \cite{SST-5}, and IMDB v1 \cite{IMDB}), demonstrating that \texttt{ProjRes} consistently outperforms state-of-the-art MIAs and achieves near 100\% accuracy in most scenarios. Furthermore, \texttt{ProjRes} remains effective under diverse FedLLM fine-tuning strategies and active defense mechanisms, highlighting its robustness and generalizability. These results offer new insights and practical methodologies for assessing the privacy risks of FedLLMs. The specific contributions are as follows:
	
	\begin{icompact}
		\item We investigate the privacy risks of FedLLMs and propose \texttt{ProjRes}, an efficient MIA that exploits gradient residual projection information to reveal sample participation.
		
		\item We reformulate membership inference in FedLLMs as a residual projection matching problem within the gradient semantic space, enabling efficient inference based on gradients from trainable components, \eg LoRA and adapter modules, as well as model substructures including attention heads and feedforward layers.
		
		\item Through theoretical analysis and empirical validation, we demonstrate that the semantic representations (embeddings) of samples in high-dimensional LLMs can be reversibly analyzed via gradient information, with increasing dimensionality amplifying privacy leakage risks.
		
		\item Extensive experiments across multiple benchmark datasets show that \texttt{ProjRes} achieves state-of-the-art performance, attaining near-perfect ($ \approx 100\%$) inference accuracy in most scenarios.
		
	\end{icompact}
	
	\section{Background and Threat Model}
	
	
	\subsection{Federated Learning}
	FL is a distributed machine learning framework that enables multiple clients to collaboratively train a model by exchanging training-related information, such as gradients, without sharing their private data. In this work, we adopt the FedSGD \cite{FL} setting. After the server initializes the global model $f_\theta$ parameterized as $\theta$, clients collaboratively participate in iterative training under the server’s coordination. In the \( t \)-th training iteration, the \( k \)-th client computes the gradient $\nabla \mathcal{L}(\theta_k^t, B_k^t)$ of the model parameters $ \theta_k^t $ based on the loss function $ \mathcal{L} $ and its local data batch $ B_k^t $. Each client then transmits its computed gradient to the server. The server aggregates the received gradients from all clients and updates the model parameters $ \theta^{t+1} $ by subtracting the aggregated gradient from the current parameters $ \theta^t $:
	\begin{equation}
		\theta^{t+1} = \theta^t - \frac{\eta}{K} \sum_{k=1}^{K} \nabla_{\theta_k^t} \mathcal{L}(\theta^t, B_k^t),
		\label{eq:theta}
	\end{equation}
	where $ K$ is the number of clients, $ \eta $ is the learning rate. The server then sends the updated parameters $ \theta^{t+1} $ to each client for updating their local models, resulting in the new global model $ \theta_k^{t+1} $. FL then repeats the above steps until the global model converges. In this way, FL enables collaborative learning between different clients without sharing data, which alleviates privacy concerns to some extent.
	
	\subsection{Federated Large Language Model}
	FedLLMs refer to the integration of FL and LLMs to address the critical challenges of data silos and limited computational resources in LLM training and fine-tuning, as shown in Fig. 1. With the growing popularity of fine-tuning methods, FedLLMs in this context primarily denote the use of the FL framework for collaboratively fine-tuning LLMs. However, due to the inherent distinctions between LLMs and traditional machine learning models, FedLLMs exhibit significant differences from conventional FL, particularly in global model initialization and the configuration of trainable parameters. The general execution process of FedLLMs can be summarized as follows:
	
	\textit{Step 1: Global Model Initialization.} 
	In conventional FL, the server typically initializes the global model with random parameters. In contrast, in FedLLMs, the initial global model $ f_{\theta^0} $ is usually instantiated from a pre-trained LLM, \ie $\theta^0  \gets \theta^{pre}$, where $ \theta^{pre} $ denotes the parameters of the pretrained LLM. Each client downloads the global model parameters $\theta^0$ before local fine-tuning.
	
	\textit{Step 2: Trainable Parameter Configuration and Local Training.}
	Unlike conventional FL, which performs full-model training, FedLLMs limit the trainable parameters to reduce resource costs and improve communication efficiency. The model parameters are divided into:
	\begin{equation}
		\theta = \theta_{\text{trainable}} \cup \theta_{\text{frozen}},
		\quad \text{where } \theta_{\text{trainable}} \subset \theta.
	\end{equation}
	That is, clients only update $\theta_{\text{trainable}}$ while keeping $\theta_{\text{frozen}}$ fixed. Depending on the configuration, FedLLMs can adopt several fine-tuning strategies, \eg, Partial Parameter Fine-Tuning~\cite{lin_2024, 10097124}, Architectural Fine-Tuning~\cite{adapter_structure, adapter_efficent, adapter_Heterogeneous, adapter_Dual-Personalizing, adapter}, and Prefix Fine-Tuning~\cite{Parameter-Efficient_Prompt_Tuning, P-tuningV2}. More details can be found in the Appendix~\ref{finetuning}.

	\textit{Step 3: Aggregation.} After local fine-tuning, clients send updates of their trainable parameters $\Delta \theta_k$ (which may include LoRA parameters, adapters, or prefixes) to the central server for aggregation. 
	
	While the FedLLMs described above effectively address the issues of data silos and insufficient computing power in collaborative fine-tuning of LLMs across different clients, it remains vulnerable to privacy inference attacks (\ie MIAs)~\cite{ICLR_MIA, ref_MIA2}. Therefore, we subsequently define our threat model and present existing MIA attacks.
	
	\subsection{Threat Model}
	Similar to prior work on MIAs~\cite{ICLR_MIA, ref_MIA2} in FL, \texttt{ProjRes} adopts a practical threat model in which an adversary seeks to conduct privacy inference attacks based on interactive gradient information, as shown in Fig. \ref{P-MIA-overview}. To avoid redundancy, we use adapter-fine-tuned FedLLMs as representative examples in the subsequent analysis and algorithmic derivation. We then elaborate on the defined threat model, detailing the adversary’s objectives, background knowledge, and capabilities.

	

	\para{Adversary's Objective.}
	Following prior work~\cite{ICLR_MIA, ref_MIA2}, the adversary’s objective is to determine whether a given sample $\bm x$ was included in the local training data of the $k$-th client during its current training round. To this end, the adversary constructs a membership inference classifier $CF(\cdot)$ that outputs a binary prediction indicating the membership status of $\bm x$:
	\begin{equation}\label{eq-11}
		CF(\bm x) =
		\begin{cases}
			1, & \text{if } \bm x \in \mathcal{D}_k^{\text{train}}, \\
			0, & \text{otherwise},
		\end{cases}	
	\end{equation}
	where $\mathcal{D}_k^{\text{train}}$ denotes the dataset used by client $k$ for the current local training epoch. An attack is considered successful if and only if $CF(\bm x) = 1$ for a sample $\bm x$ that exactly matches an instance in $\mathcal{D}_k^{\text{train}}$. If $CF(\bm x) = 1$ due to mere semantic resemblance, lexical overlap, or partial token similarity with samples in the training set, the prediction does not constitute a valid membership inference and is therefore regarded as unsuccessful.

	
	\para{Adversary's Background Knowledge.} The adversary can be any entity attempting to infer private information about clients’ local data, including but not limited to a privacy auditor, an honest-but-curious server, honest-but-curious clients, or a competing participant in the FedLLM. Each type of adversary may gain access to different sources of information depending on its role in the system. For instance, a privacy auditor may observe model updates for security evaluation; a malicious client may exploit shared global parameters; and a central server may analyze intermediate gradients uploaded by clients. In this work, and consistent with the predominant threat model in FedLLMs, we assume an honest-but-curious server as the adversary. This server follows the prescribed training protocol but passively inspects the information exchanged during training, such as model updates, intermediate gradients, or loss trajectories, to infer private attributes or determine the membership of specific samples. 
	
	
	\para{Adversary's Capabilities.} The adversary has access to the global model parameters $\theta$ and the gradients $\nabla \theta$ corresponding to the trainable parameters uploaded by the clients. Since the honest-but-curious server assumes the role of the adversary, this information lies within its legitimate authority and represents naturally obtainable intermediate information. Beyond these, the attacker has no access to any private knowledge of the target client, including local training samples, data distributions, or other sensitive information associated with the client’s local data. This restriction reflects a realistic semi-honest threat model, which is commonly adopted in the security analysis of FedLLMs.

	\subsection{Taxonomy of MIAs}
	\label{Baseline}
	Within the context of the threat model described above, prior work has primarily focused on four categories of MIAs in FL: loss-based MIAs, gradient orthogonality-based MIAs, training dynamics-based MIAs, and statistical discrepancy–based MIAs. Our analysis builds on the observation that a straightforward approach to membership inference in FedLLMs is to directly adapt these existing MIAs.
	
	Building upon Eq. \eqref{eq-11}, the adversary typically uses proxy information \(I\) (\eg gradients) to calculate the membership score $\psi(\bm x,I)$ of a sample \(\bm x\), and then determines whether $\bm x$ belongs to the client's local training dataset by comparing the score with a predefined threshold $\tau$. The classifier $CF(\cdot,\cdot)$ in Eq. \eqref{eq-11} can be rewritten as follows:
	\begin{equation}
		CF(\bm x, I)=
		\begin{cases}
			1, & \mathrm{if} \quad \psi(\bm x, I) > \tau, \\
			0, & \text{otherwise},
		\end{cases}.
		\label{eq:attack_formulation}
	\end{equation}
	In Eq. \eqref{eq:attack_formulation}, different choices of the function $ \psi(\bm x, I) $ or decision threshold $ \tau $ lead to distinct MIA strategies. We formally introduce them as follows:
	\begin{figure*}[!t]
		\centering
		\includegraphics[width=1\textwidth]{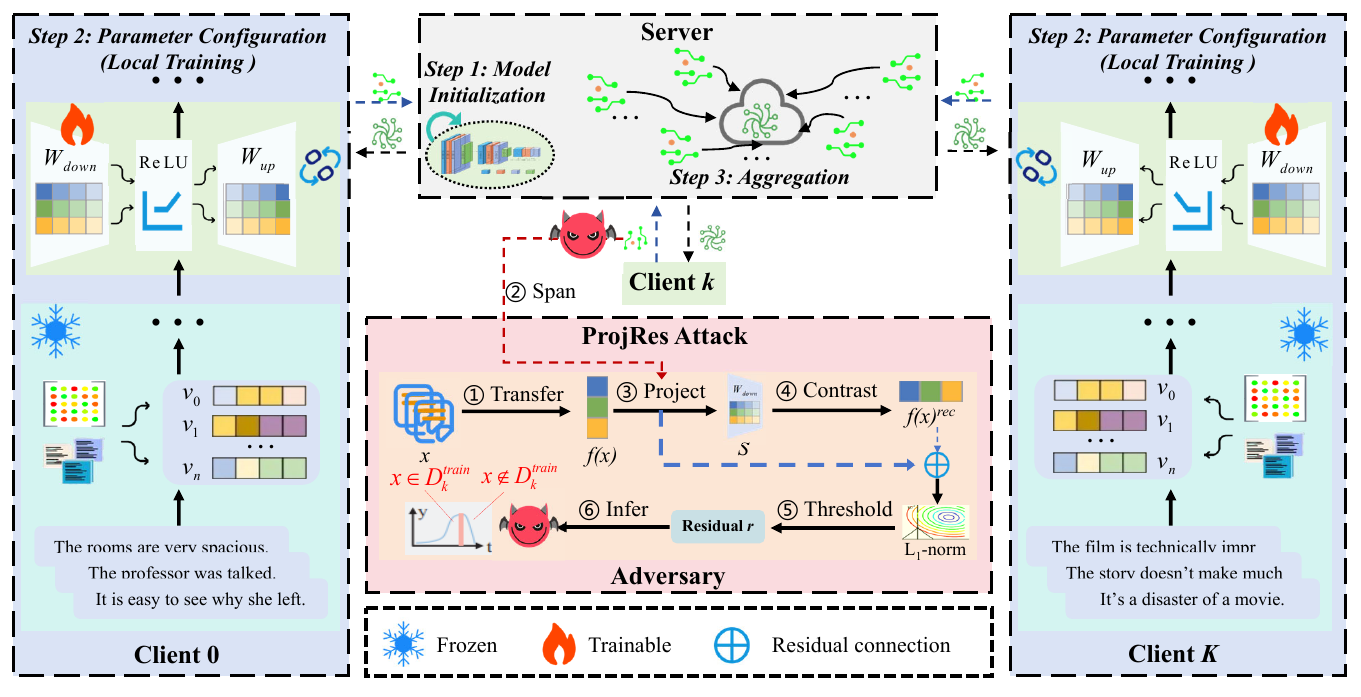} %
		\caption{Overview of the proposed \texttt{ProjRes} attack. Using the adapter as the trainable module, an honest-but-curious server conducts the \texttt{ProjRes} attack as follows:  
			\ding{172} Perform a forward pass on a target sample $\bm x$ to obtain its hidden embedding $f(\bm x)$ at the adapter layer, where $f(\cdot)$ denotes the model segment preceding the adapter.  
			\ding{173} Construct a linear subspace $\mathcal{S}$ spanned by the adapter downsampling-layer gradients uploaded by clients.  
			\ding{174} Projection: Compute the orthogonal projection of $f(\bm x)$ onto $\mathcal{S}$ to obtain $f(\bm x)^{\text{rec}}$.  
			\ding{175} Residual computation: Calculate the difference $f(\bm x) - f(\bm x)^{\text{rec}}$.  
			\ding{176} Residual measurement: Compute the residual magnitude $r = \| f(\bm x) - f(\bm x)^{\text{rec}} \|_1$.  
			\ding{177} Membership inference: If $r < \tau$, infer $\bm x$ as a member; otherwise, classify it as a non-member.}
		\label{P-MIA-overview}
		\vspace{-0.3cm}
	\end{figure*} 
	
	\para{Loss-based MIAs.} The representative methods include FedLoss \cite{ICLR_MIA}, Score-Diff \cite{ref_MIA2}, and Score-Ratio \cite{ref_MIA2}. These approaches are grounded in the key observation that the loss of a training sample on the model is typically significantly lower than that of a non-training sample. Consequently, membership can be effectively inferred by comparing the magnitude of the sample's loss. 
	\begin{equation}
		\psi(\bm x, I) = \mathcal{A}\big( \mathcal{L}(\theta, \bm x) \big),
	\end{equation}
	where \( \mathcal{A}(\cdot) \) denotes the predefined inference function.
	
	\para{Gradient Orthogonality-based MIAs.} Typical methods include Cosine Similarity \cite{ICLR_MIA} and Gradient-Diff \cite{ICLR_MIA}. These approaches leverage the observation that, as training converges, gradients generated by different samples tend to become nearly orthogonal. Under this condition, the projection of a member sample's gradient onto the true client's gradient exhibits a statistically distinguishable pattern compared to that of a non-member, enabling effective MIA.
	\begin{equation}
		\psi(\bm x, I) = \mathcal{A}\big( \nabla_{\theta_k} \mathcal{L}(\theta, \bm x) \big).
	\end{equation}
	
	\para{Training Dynamics–based MIAs.} Exemplified by FTA \cite{FTA}, this class of methods assumes that during local fine-tuning, member samples induce faster or more pronounced changes in model performance metrics, such as loss or accuracy, compared to non-member samples. Across multiple training rounds, these differences give rise to distinct temporal patterns, which can be leveraged to differentiate member samples from non-members.
	\begin{equation}
		\psi(\bm x, I) = \mathcal{A}( \theta_k^0, \theta_k^1, \ldots, \theta_k^t; \bm x).
	\end{equation}
	
	\begin{table}[t!]
		\centering
		\caption{Comparison of existing MIAs and \texttt{ProjRes}: key assumptions and applicability in FedLLMs.}
		\label{tab:mia_comparison}
		\renewcommand{\arraystretch}{1.6}
		\begin{adjustbox}{width=0.49\textwidth, center}
			\begin{tabular}{lcccc}
				\toprule
				\textbf{Method} &
				\makecell{\textbf{Requires} \\ \textbf{Shadow Models}} &
				\makecell{\textbf{Relies on} \\ \textbf{Multi-Round Dynamics}} &
				\makecell{\textbf{Assumes Gradient} \\ \textbf{Orthogonality}} &
				\makecell{\textbf{Effective in} \\ \textbf{FedLLMs}} \\
				\midrule
				Loss-based~\cite{ICLR_MIA,ref_MIA2}        & \xmark & \cmark & \xmark & \xmark \\
				Gradient Orthogonality-based~\cite{ICLR_MIA} & \xmark & \cmark & \cmark & \xmark \\
				Training Dynamics-based~\cite{FTA}         & \cmark & \cmark & \xmark & \xmark \\
				Statistical Discrepancy-based~\cite{FedMIA} & \cmark & \cmark & \xmark & \xmark \\
				\midrule
				\texttt{ProjRes} (Ours)                    & \xmark & \xmark & \xmark & \cmark \\
				\bottomrule
			\end{tabular}
		\end{adjustbox}
		\begin{tablenotes}
			\footnotesize
			\item  	\footnotesize\textit{Note: ``Effective in FedLLMs'' means the method remains practical and accurate under PEFT, high representation similarity, and single-round gradient exposure.}
		\end{tablenotes}   
		\vspace{-0.3cm}
	\end{table}
	
	\para{Statistical Discrepancy–based MIAs.} A representative method is FedMIA \cite{FedMIA}. It assumes that, for the target client $k$, member samples appear exclusively in client’s local training data and are not held by any other client. Under this assumption, the gradients uploaded by the remaining clients can be used to model the gradient distribution of non-members. The attacker then compares the gradient induced by a target sample to this modeled non-member distribution, if the sample’s gradient deviates sufficiently from the non-member distribution, the attacker infers that the sample is a member of client $k$’s training set.
	\begin{equation}
		\psi(\bm x, I) = \mathcal{A}( \theta_0^t, \ldots, \theta_{k-1}^t, \theta_{k}^t, \ldots, \theta_K^t;\bm x).
	\end{equation}
	
	\para{Discussion.} However, directly applying existing MIAs to FedLLMs faces several intrinsic challenges (see Table \ref{tab:mia_comparison}). First, the cosine similarities between hidden embeddings in LLMs are tightly concentrated within $[0.91,1]$ (to be discussed later), indicating high representation homogeneity that undermines loss-based distinguishability. Second, this similarity leads to non-orthogonal gradients (see Fig.~\ref{fig:cosine_sim}), making gradient-orthogonality–based MIAs ineffective. Third, under the FedSGD setting, clients only upload gradients from the current mini-batch, preventing attackers from exploiting training dynamics. Finally, statistical discrepancy–based MIAs, which rely on distributional differences between member and non-member gradients, lose efficacy when such gaps are minimal.
	
	\section{Our Proposed \texttt{ProjRes}}
	
	\subsection{Overview}
	\para{Key Idea.} In this section, we present \texttt{ProjRes}, an efficient MIA tailored for FedLLMs, as shown in Fig.~\ref{P-MIA-overview}. \texttt{ProjRes} consists of three core modules, \ie Sample Representation Construction, Semantic Space Mapping, and Residual Computation and Membership Decision, which proceeds through six concrete execution steps. Unlike prior approaches, \texttt{ProjRes} does not depend on traditional signals such as loss values \cite{ICLR_MIA}, prediction confidence scores \cite{ref_MIA2}, or gradient norms \cite{ICLR_MIA}, which are often unavailable or unreliable in FedLLM deployments. Instead, our core insight is to exploit the algebraic relationship between a sample’s hidden embeddings and the gradients the client uploads: by mapping embeddings into a semantic space aligned with the gradient subspace and analyzing the residuals, \texttt{ProjRes} can effectively distinguish member from non-member samples. Next, we will briefly describe each key module and its execution steps as follows.
	
	\para{Sample Representation Construction (\S\ref{sec-3-2}).} This module aims to extract the embeddings of target samples during forward propagation (\ie perform step \ding{172}), thereby analyzing the cosine similarity between the hidden embedding vectors of different samples and the $\ell_1$ norm of their pairwise differences. This provides a semantic information foundation for membership inference.
	
	\para{Semantic Space Mapping (\S\ref{sec-3-3}).}
	This module corresponds to steps \ding{173}–\ding{174} of the \texttt{ProjRes} procedure. The server first constructs a semantic subspace $\mathcal{S}$ based on the gradient vectors of the adapter’s downsampling layer collected from participating clients. Each gradient vector represents the local optimization direction induced by client-specific data, implicitly encoding the underlying semantic distribution. By performing an orthogonal projection of the target embedding $f(\bm x)$ onto $\mathcal{S}$, the server obtains its reconstructed representation $f(\bm x)^{\text{rec}}$, which captures how well the target sample aligns with the collective subspace of client data.
	
	\para{Residual Computation and Membership Inference (\S\ref{sec-3-4}).}
	In this stage, the residual vector between the original and reconstructed embeddings is computed as $f(\bm x) - f(\bm x)^{\text{rec}}$, and its $\ell_1$-norm $r = \| f(\bm x) - f(\bm x)^{\text{rec}} \|_1$ serves as the residual magnitude (\ie perform steps \ding{175}-\ding{176}). A smaller residual indicates a stronger alignment between the target sample and the client’s learned subspace, implying a higher likelihood of membership. Finally, by comparing $r$ against a predefined threshold $\tau$, the server determines the membership status of $\bm x$, completing the inference process (\ie perform step \ding{177}).

	Furthermore, in Appendix \ref{sec-3-5}, we provide a theoretical analysis of \texttt{ProjRes} regarding the upper bound of inference capabilities and privacy leakage. Notably, \texttt{ProjRes} requires only the gradient update $ \nabla{\theta_k^t} $ uploaded by the target client in a single communication round and does not rely on updates from other clients. Consequently, in the subsequent theoretical analysis, we omit the round index \(t\) and client identifier $k$, focusing instead on the relationship between the global model $f_\theta$ and the corresponding gradients $ \nabla{\theta} $.
	
	\subsection{Sample Representation Construction}\label{sec-3-2}
	\para{Key Challenge.} Historically, a crucial factor in the success of MIAs has been the development of effective sample representations, which enable the use of surrogate information (\eg gradients or losses) to infer membership. Similarly, in the context of FedLLMs, a fundamental challenge lies in determining how to represent samples effectively, that is, identifying what information best captures their intrinsic semantic and structural characteristics for reliable inference.

	\para{Key Idea.} To recap: In a FedLLM, the LLM essentially functions as an embedding model, whose core role is to map textual semantic information into high-dimensional vector representations. Through this mapping, semantically similar texts are positioned close together in the embedding space, while semantically dissimilar texts are placed farther apart—thereby transforming abstract semantic relationships into a computable vector space. This process implicitly relies on a key assumption: \textit{different input samples correspond to distinct hidden embedding vectors.} Building on this intuition, our core idea is to leverage these hidden embeddings as sample representations, enabling efficient and fine-grained membership inference.

	To validate this assumption, we conducted experiments using the BERT-Base model on the SST-2 dataset, focusing on the cosine similarities between different samples’ hidden embedding vectors and the $\ell_1$-norms of their pairwise differences. For visualization, Fig. \ref{fig:cosine_parameter} presents the cosine similarity matrix for 50 randomly selected samples, while Fig. \ref{fig:L1_Distance} depicts the $\ell_1$-norms of the pairwise differences between their embedding vectors. This visualization provides a clear comparison of the relative similarities and distinctions among the samples’ hidden representations.

	\begin{figure}[t]
		\centering
		\begin{subfigure}[t]{0.23\textwidth}
			\includegraphics[width=\textwidth]{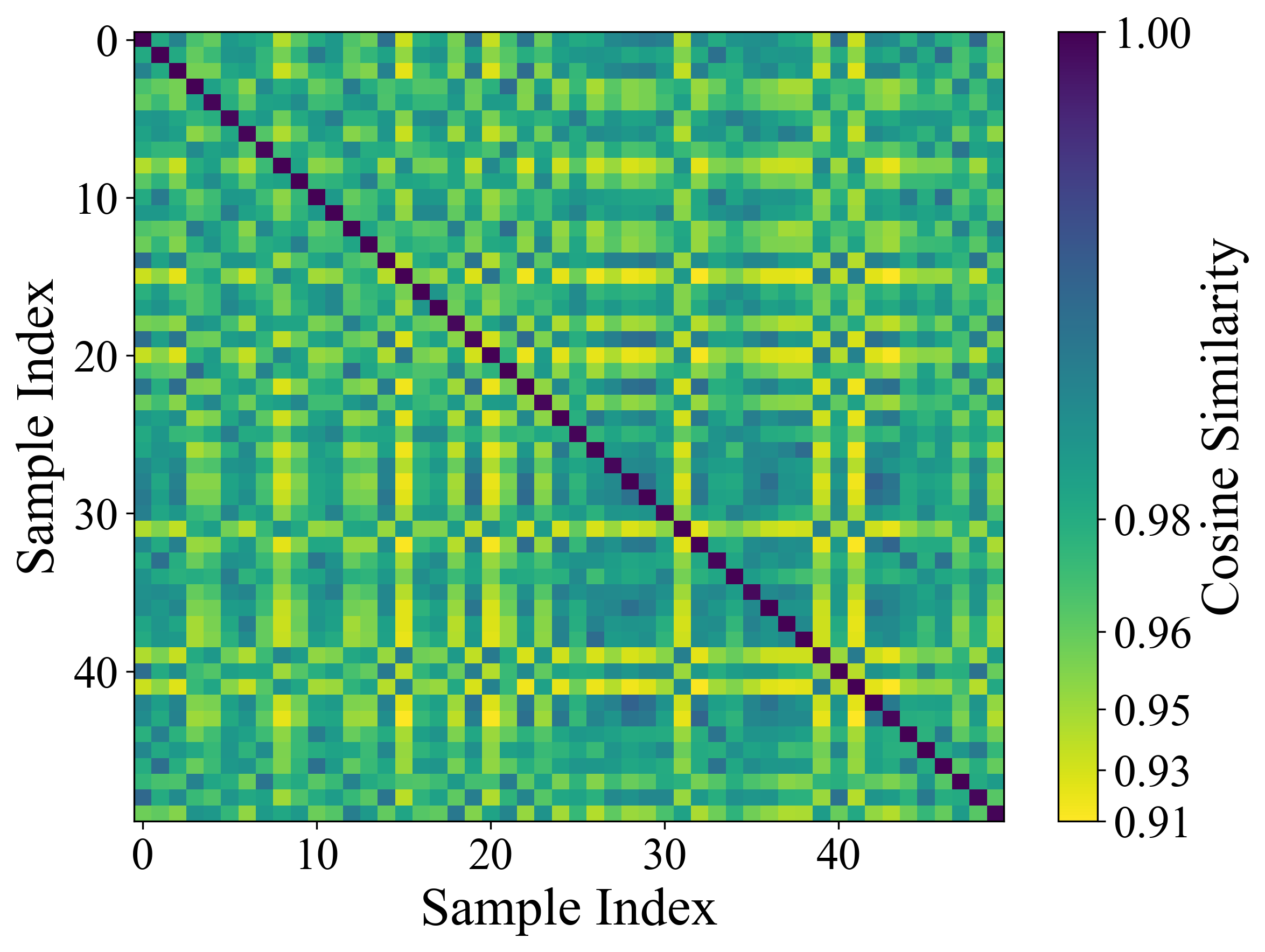}
			\caption{Cosine similarity.}
			\label{fig:cosine_parameter}
		\end{subfigure}
		\hfill
		\begin{subfigure}[t]{0.23\textwidth}
			\includegraphics[width=\textwidth]{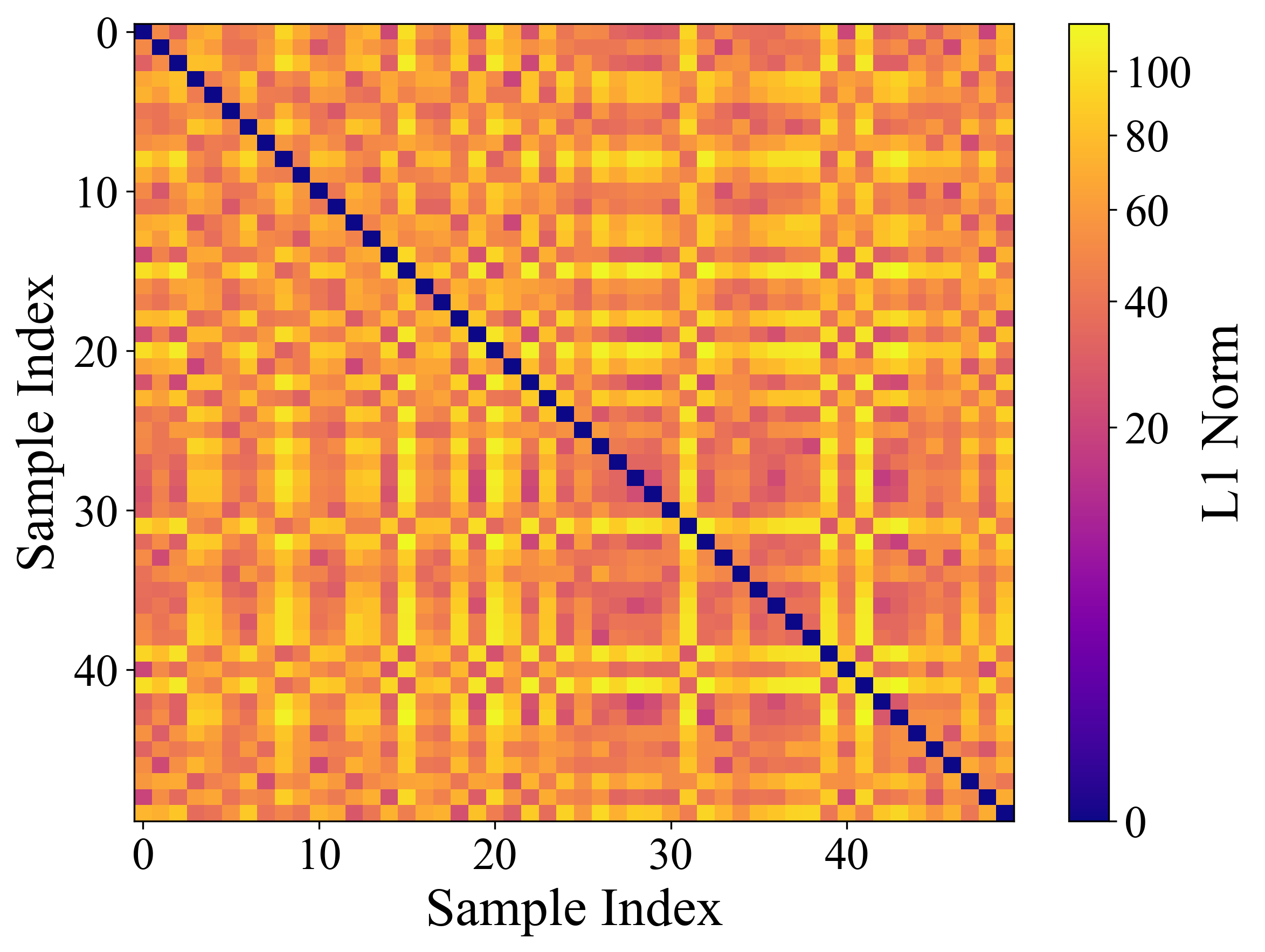} 
			\caption{$\ell_1$ norms.}
			\label{fig:L1_Distance}
		\end{subfigure}
		\caption{Comparison of hidden embedding vector properties across different samples. (a) Cosine similarity comparison. (b) $\ell_1$ norms of vector differences comparison.}
		\label{fig:Similarity_Distance}
		\vspace{-0.5cm}
	\end{figure}
	
	From Fig.~\ref{fig:Similarity_Distance}, two key observations can be drawn. First, there exist discernible differences in the directions of hidden embedding vectors across different samples. As illustrated in Fig.~\ref{fig:cosine_parameter}, although the cosine similarities between hidden embeddings generally fall within the range of 0.91 to 1, each vector achieves a cosine similarity of 1 only with itself. This suggests that while the embedding vectors share broadly similar orientations, they are not identical. Second, there are substantial differences in the magnitudes of hidden embedding vectors, as reflected in Fig.~\ref{fig:L1_Distance}. The $\ell_1$-norms of pairwise differences between embeddings are relatively large—an effect amplified by the high dimensionality of LLM embeddings, which accentuates even subtle distinctions among samples.


	These observations indicate that both the directional information and the $\ell_1$-norm magnitudes of hidden embedding vectors can effectively characterize samples, which aligns with intuitive understanding. Specifically, the directional information captures angular distinctions between embeddings, compensating for cases where $\ell_1$-norms alone fail to distinguish samples with similar orientations. Conversely, the $\ell_1$-norm provides a complementary measure of magnitude differences, addressing limitations of relying solely on directionality. By integrating these two aspects, we achieve a more comprehensive representation, enabling more accurate description and differentiation of samples, enhanced feature capture, and improved discrimination precision, particularly in complex datasets. While obtaining hidden embedding vectors is straightforward, requiring only a forward pass through the model (see Step \ding{172}), effectively leveraging both their directional and magnitude information remains a key technical challenge in our design.

	
	
	\subsection{Semantic Space Mapping} 
	\label{sec-3-3}
	\para{Key Challenge.} Once we obtain a sample’s representation, we need to establish how that representation semantically relates to model gradients, which is a necessary prerequisite for gradient-based membership inference~\cite{MIA_based_on_Grade1}. Gradients, as update directions, implicitly encode the learning signal induced by training samples and thus carry semantic information in the model’s high-dimensional parameter space~\cite{MIA_based_on_ReLU}. However, this gradient space differs fundamentally from the embedding (semantic level) space in both structure and dimensionality. The key challenge, therefore, is to learn an effective mapping from the high-dimensional gradient space back to the (comparatively low-dimensional) embedding space and to reconstruct a sample’s semantic content from gradients that only indirectly reflect it. 

	\para{Key Idea.} Gradients, which capture the model’s progression from sample-specific learning to generalization, intrinsically encode how the semantic information of input samples is transformed within the model’s parameter space. The exact form of this mapping depends on the underlying network architecture. Building on this intuition, our key idea is to investigate the intrinsic relationship between gradients and sample-level semantics by exploiting the structural properties of neural networks. As illustrated in Fig.~\ref{P-MIA-overview}, the hidden embeddings produced by a Transformer layer serve as the input to the Adapter module inserted at that layer. Owing to the principle of backpropagation, the gradients of the Adapter inherently carry information derived from these input embeddings. In the subsequent sections, we detail how this gradient-embedded semantic information can be effectively extracted and utilized for membership inference.
	
	According to \cite{adapter_structure}, an Adapter typically comprises three components, \ie a downsampling layer, a non-linear activation layer, and an upsampling layer, as illustrated in Fig.~\ref{P-MIA-overview}. The downsampling layer serves as the Adapter’s input layer and can be viewed as a generalized form of a fully connected structure, encompassing modules such as LoRA, $Q/K/V$ projection matrices, and multilayer perceptrons. Even when the input dimension is smaller than or equal to the output dimension, a comparable downsampling effect can still be achieved through output sampling techniques. For a fully connected layer, let the input be $\bm{X} \in \mathbb{R}^{p \times n}$ and the corresponding output be $\bm{Y} \in \mathbb{R}^{p \times m}$, where $m/n$ denotes the downsampling factor. Let $\bm{W} \in \mathbb{R}^{n \times m}$ be the weight matrix connecting the input and hidden layers, and $\bm{b} \in \mathbb{R}^{m}$ is the bias vector. The relationship between the input and the hidden-layer output can thus be formulated as:
	\begin{equation}	\label{eq:hidden_layer_output}
		\bm{Y} = \bm{X} \bm{W} + \bm{b}.	
	\end{equation}
	
	According to the chain rule of differentiation, the gradient of the weight parameters can be computed as follows:
	\begin{equation}
		\nabla \bm{W} = \frac{\partial \mathcal{L}}{\partial \bm{W}} = \frac{\partial \mathcal{L} / \partial \bm{Y}}{\partial \ \bm{Y} / \partial \bm{W}} = \frac{\partial \mathcal{L}}{\partial \bm{Y}} \bm{X},
		\label{eq:w_gradient}
	\end{equation}
	where \( \mathcal{L} \) denotes the loss value. Let \( {\partial \mathcal{L}}/{\partial \bm{Y}}= {\bm\alpha} \), then Eq. \eqref{eq:w_gradient} can be rewritten as:
	\begin{equation}
		\nabla \bm{W}={\bm\alpha}\bm{X}.
		\label{eq:gradient}
	\end{equation}
	By treating \( {\alpha} \) as constant coefficientes, it can be observed that the gradient of the fully connected layer weights \( \bm{W}\) is essentially a linear combination of the input data \( \bm{X} \). In other words, the gradient contains complete information about the input, which demonstrates the memory property of the weights' gradient $ \nabla \bm{W} $ in the fully connected layer. The ranks of matrices $ \nabla \bm{W} $, $ {\bm\alpha} $, and $ \bm{X} $ can be expressed as follows:
	\begin{equation}
		\operatorname{rank}(\nabla \bm{W})=\operatorname{rank}({\bm\alpha}\bm{X})=\operatorname{min}(n, p, m),
		\label{eq:rank_w}
	\end{equation}
	\begin{equation}
		\operatorname{rank}({\bm\alpha})=\operatorname{min}(m, p),
		\label{eq:rank_a}
	\end{equation}
	\begin{equation}
		\operatorname{rank}(\bm{X})=\operatorname{min}(n, p),
		\label{eq:rank_x}
	\end{equation}
	where \( \operatorname{rank}(\cdot) \) denotes the rank function. The equality $ \operatorname{rank}(\nabla \bm{W})= \operatorname{rank}(\bm{\alpha})=\operatorname{rank}(\bm{X}) $ holds if and only if $ p \leq m $ and $ p<n $. The analysis of the remaining cases is provided in Appendix \ref{Basic_Capability}. In this case, 
	\begin{equation}
		\mathcal{S} = \operatorname{Span}(\nabla \bm{W})=\operatorname{Span}(\bm{X}),
		\label{eq:Span_x_w}
	\end{equation}
	where \(\operatorname{Span}(\cdot)\) denotes the linear subspace spanned by the given vectors.
	
	Since \( p < n \), the space \( \mathcal{S} = \operatorname{Span}(\bm X) \) forms a low-rank linear subspace, specifically, whose number of basis vectors \( p \) is strictly smaller than the ambient dimension \( n \). In this setting, the probability that a vector drawn at random from a continuous distribution over \( \mathbb{R}^n \) lies exactly in \( \mathcal{S} \) is zero. Consequently, if a given vector is observed to reside within \( \mathcal{S} \), it can be inferred with high confidence that the vector is generated from (or lies in the span of) the basis of \( \mathcal{S} \). Therefore, to determine whether an arbitrary vector \(\bm{x}_i\) lies in the subspace \(\mathcal{S}\), we first compute its projection onto \(\mathcal{S}\) using the following formula, thereby transforming the information in the high-dimensional subspace \(\mathcal{S}\) into a vector that is dimensionally consistent with \(\bm{x}_i\).
	\begin{equation}
		\bm{x}_i^{\mathrm{rec}} = \Pi_{\mathcal{S}}(\bm{x}_i),
		\label{eq:x_rec}
	\end{equation}
	where \(\Pi_{\mathcal{S}}(\cdot)\) denotes the projection of a vector onto the subspace \(\mathcal{S}\), whose computation is detailed in Appendix \S\ref{full_description}. Based on Eq. \eqref{eq:x_rec}, we can obtain the projection \(f(\bm x)^{\mathrm{rec}}\) of \( f(\bm x) \) on the subspace \( \mathcal{S} \), 
	\begin{equation}
		f(\bm x)^{\mathrm{rec}} = \Pi_{\mathcal{S}}(f(\bm x)).
		\label{eq:f_rec}
	\end{equation}
 \( f(\bm x)^{\mathrm{rec}} \) also reflects the semantic information component of the hidden embedding vector \( f(\bm x) \) captured by the gradients. At this point, \texttt{ProjRes} has completed steps \ding{173}-\ding{174}.

	\subsection{Residual Computation and Membership Inference}\label{sec-3-4}
	\para{Key Idea.} From the above, we can see that \( f(\bm x)^{\mathrm{rec}} \) fundamentally quantifies the extent to which the semantic content of the hidden embedding \( f(\bm x) \) contributes to the model gradients. This contribution is inherently governed by the membership status of the sample \( \bm x \): if \( \bm x \) belongs to the training data, its semantic information should be adequately captured in the gradients, resulting in high similarity between the reconstructed embedding \( f(\bm x)^{\mathrm{rec}} \) and the original embedding \( f(\bm x) \); in contrast, for non-member samples, this alignment is significantly diminished. Consequently, effectively exploiting the intrinsic relationship between \( f(\bm x) \) and \( f(\bm x)^{\mathrm{rec}} \) to substantially enhance the distinguishability between member and non-member samples in this semantic space is our key idea.
	
	As demonstrated in \S\ref{sec-3-2}, the hidden embedding vectors of different samples exhibit highly aligned directions in the vector space, indicating that direction-based metrics, such as cosine similarity, are insufficient for reliably distinguishing between member and non-member samples. Moreover, the per-dimension differences in these embeddings are extremely small: for a 768-dimensional hidden embedding, the reconstruction error has an \( \ell_1 \) norm of approximately 100, corresponding to an average absolute difference of only about 0.13 per dimension. Such small discrepancies are further attenuated under higher-order norms (\eg \( \ell_2 \), \( \ell_n \) with \( n > 1 \)), which suppress the discriminative signal between member and non-member samples. Therefore, we adopt the \( \ell_1 \) norm as our final evaluation metric, as it effectively accumulates these subtle per-dimension deviations across the high-dimensional space, thereby significantly amplifying the distinguishability between member and non-member samples. To this end, the residual $ r $  can be calculated as the following equation:
	\begin{equation}
		r=\|f(\bm x) - f(\bm x)^{\mathrm{rec}}\|_1.
		\label{eq:residual}
	\end{equation}
	
	Therefore, in light of the preceding analysis, for the hidden embedding \( f(\bm x) \) of any real sample \( \bm x \), if the residual \( r \) is sufficiently small, we can conclude that \( f(\bm x) \) and \( f(\bm x)^{\mathrm{rec}} \) are sufficiently similar, \ie \( f(\bm x) \) can be accurately reconstructed within the linear subspace \( \mathcal{S} \). This implies that the semantic information encoded in the hidden embedding \( f(\bm x) \) has been fully captured by the model gradients, \( f(\bm x) \) must have served as an input to the Adapter module during local training, implying that \( \bm x \) belongs to the client’s training data. Conversely, if \( f(\bm x) \) cannot be well reconstructed in \( \mathcal{S} \), then \( \bm x \) was not used as an Adapter input and is thus not part of the local training set. The membership of \( \bm x \) can be expressed as follows:
	\begin{equation}\label{eq-27}
		CF(\bm x,f(\bm x)^{\mathrm{rec}}) =
		\begin{cases}
			1, & \text{if }  r < \tau, \\
			0, & \text{otherwise},
		\end{cases}	
	\end{equation}
	where $\tau$ is the threshold, typically set to a value below $10^{-2}$. Steps \ding{175}-\ding{177} of \texttt{ProjRes} have been completed. The algorithmic procedure of \texttt{ProjRes} is as Algorithm~\ref{alg:P-MIA} in Appendix \ref{full_description}.
	
	Two key observations are noteworthy:
	First, as shown in \S\ref{sec-3-2}, the hidden embeddings $f(\bm x)$ of different samples exhibit distinct directional properties, ensuring that the reconstructed embedding \( f(\bm x)^{\mathrm{rec}} \) cannot be mistaken for that of any other sample.
	Second, the pairwise differences between embeddings of distinct samples display large $\ell_1$-norm values, resulting in a clear separation in the projection residuals $r$ between member and non-member samples. This separation forms a strong and reliable basis for accurate membership inference.

	\para{Extend to other fine-tuning strategies.} In the algorithmic derivation, \S\ref{sec-3-2} and \S\ref{sec-3-4} are based on general assumptions that characterize the intrinsic properties of the hidden embedding $f(\bm x)$, thereby ensuring broad applicability of their conclusions. In contrast, \S\ref{sec-3-3} focuses on the specific case of the downsampling layer within an Adapter module, with all derivations grounded in Eq.~\eqref{eq:hidden_layer_output}, that is \( \bm{Y} = \bm{X}\mathbf{W} + \bm{b}\), where $\bm{X} \in \mathbb{R}^{p \times n}$ denotes the input, $\bm{Y} \in \mathbb{R}^{p \times m}$ the output, and $\bm{W} \in \mathbb{R}^{n \times m}$ and $\bm{b} \in \mathbb{R}^{m}$ represent the learnable weight matrix and bias term, respectively. This formulation naturally leads to a key insight: within PEFT-based FedLLM frameworks, any trainable module whose input–output relationship follows the linear mapping in Eq.~\eqref{eq:hidden_layer_output} can potentially serve as an attack surface for \texttt{ProjRes}, since its weight gradients may encode recoverable semantic information about training samples. Typical vulnerable modules in FedLLMs therefore include, but are not limited to:
	\begin{icompact}
		\item LoRA modules in LoRA-based fine-tuning,
		\item trainable attention mechanisms,
		\item trainable fully connected layers within trainable FFNs.
	\end{icompact}
	
	\section{Evaluation}
	
	
	\subsection{Experimental Setup}
	
	\para{Datasets.} We evaluate on four text datasets: CoLA-1.1 (CoLA) \cite{CoLA}, Yelp-tip (Yelp) \cite{yelp}, SST-5 (SST) \cite{SST-5}, and IMDB v1 (IMDB) \cite{IMDB}. These datasets cover grammar verification, coarse/fine-grained sentiment analysis, and are widely used in machine learning privacy risk studies \cite{MIA_based_on_Grade1, Acitvate_MIA_FLLMs, DAGER}. The basic information of the datasets is shown in Table \ref{Dataset_Information}.
	
	\begin{table}[!t]
		\renewcommand{\arraystretch}{1.2}
		\setlength{\tabcolsep}{4pt}
		\centering
		\caption{Dataset information.}
		\label{Dataset_Information}
		\begin{threeparttable}
			{\begin{tabular}{cccccc}	\toprule
					\textbf{Datasets}   & \textbf{\# of Samples} &\textbf{\# of Categorie}s   & $L_{min}$   & $L_{max}$  & $L_{ave}$            \\\midrule
					CoLA       &  8551       & 2   &  2       & 42      &   7.7             \\
					Yelp       &  1320761    & 10   &  1       & 113     &   11.41           \\
					SST        &  8544       & 5   &  2       & 52      &   19.14           \\ 
					IMDB       &  25000      & 3    &  10      & 2470    &   233.79          \\\bottomrule 
			\end{tabular}}{} 
			\begin{tablenotes}
				\footnotesize
				\item ``\# of'' denotes the term ``Number of''.
				\item  $L$ denotes text length, which represents the unit in characters.
			\end{tablenotes}      
		\end{threeparttable}
		\vspace{-0.2cm}
	\end{table}
	
	\para{Models.} We evaluate \texttt{ProjRes} on four pre-trained models, \ie BERT-Base \cite{BERT}, GPT2-Large \cite{GPT2}, Llama3-8B \cite{llama3}, and Qwen2.5-14B \cite{qwen2}. The architectural specifications of the selected models are presented in Table \ref{Model_Information}, detailing three core technical characteristics: positional encoding implementation, type of attention mechanism, and hidden embedding dimensions. These parameters constitute key discriminators among different LLMs.
	
	\begin{table}[!t]
		\renewcommand{\arraystretch}{1.2}
		\setlength{\tabcolsep}{4pt}
		\centering
		\caption{Model information.}
		\label{Model_Information}
		\begin{threeparttable}
			\begin{tabular}{lccc}
				\toprule
				\textbf{Model}           & \makecell{\textbf{Positional}\\\textbf{Encoding}}  & \makecell{\textbf{Attention}\\\textbf{Mechanism}}     & \makecell{\textbf{Hidden Embedding}\\\textbf{Dimension}} \\
				\midrule
				BERT-Base       & \checkmark             & Masked attention        & 786    \\
				GPT2-Large      & $\times$               & Causal attention        & 1280   \\
				Llama3-8B       & $\times$               & Causal attention        & 4096   \\
				Qwen2.5-14B     & $\times$               & Causal attention        & 5120   \\
				\bottomrule 
			\end{tabular}
		\end{threeparttable}
		\vspace{-0.5cm}
	\end{table}
	
		\begin{figure*}[!t]
		\centering
		
		\begin{subfigure}[t]{0.23\textwidth}
			\includegraphics[width=\textwidth]{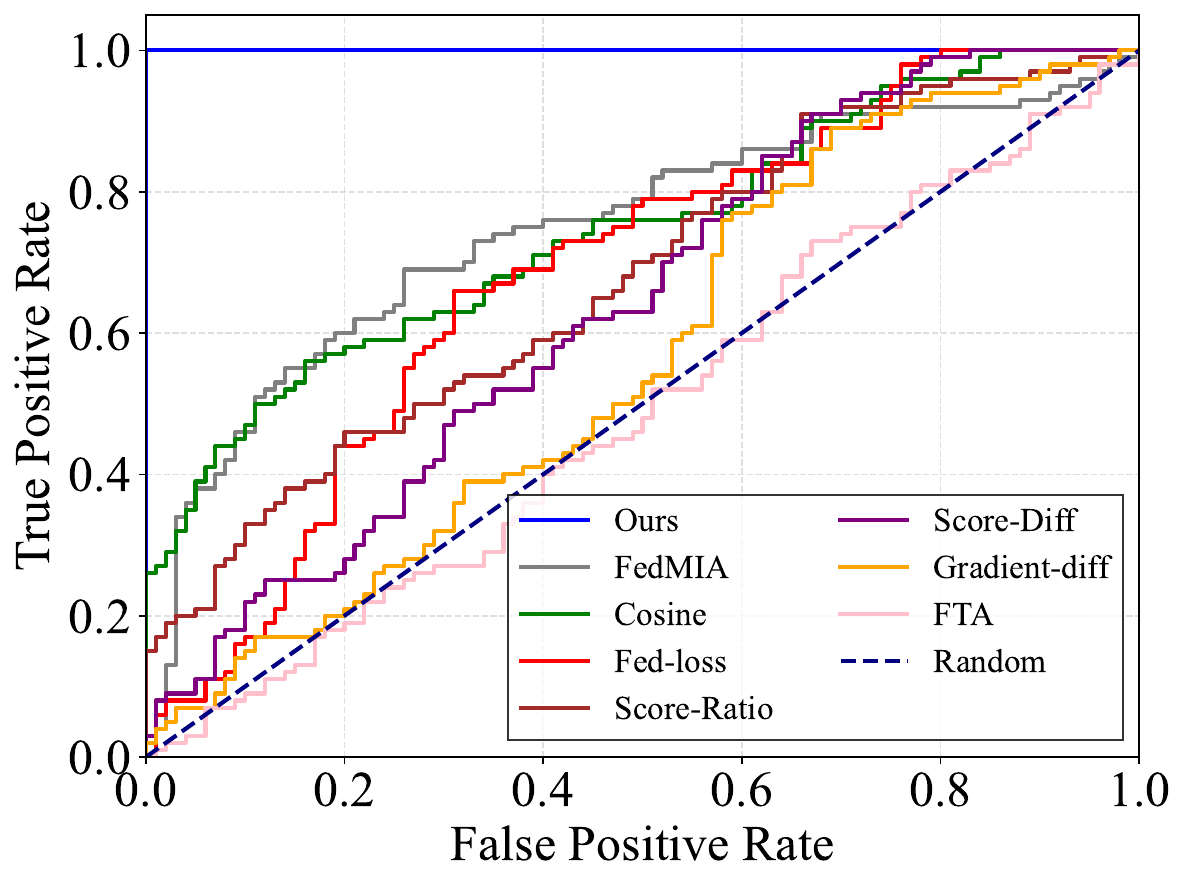}
			\caption{BERT-Base on CoLA.}
			\label{fig:bert_cola}
		\end{subfigure}
		\hfill
		\begin{subfigure}[t]{0.23\textwidth}
			\includegraphics[width=\textwidth]{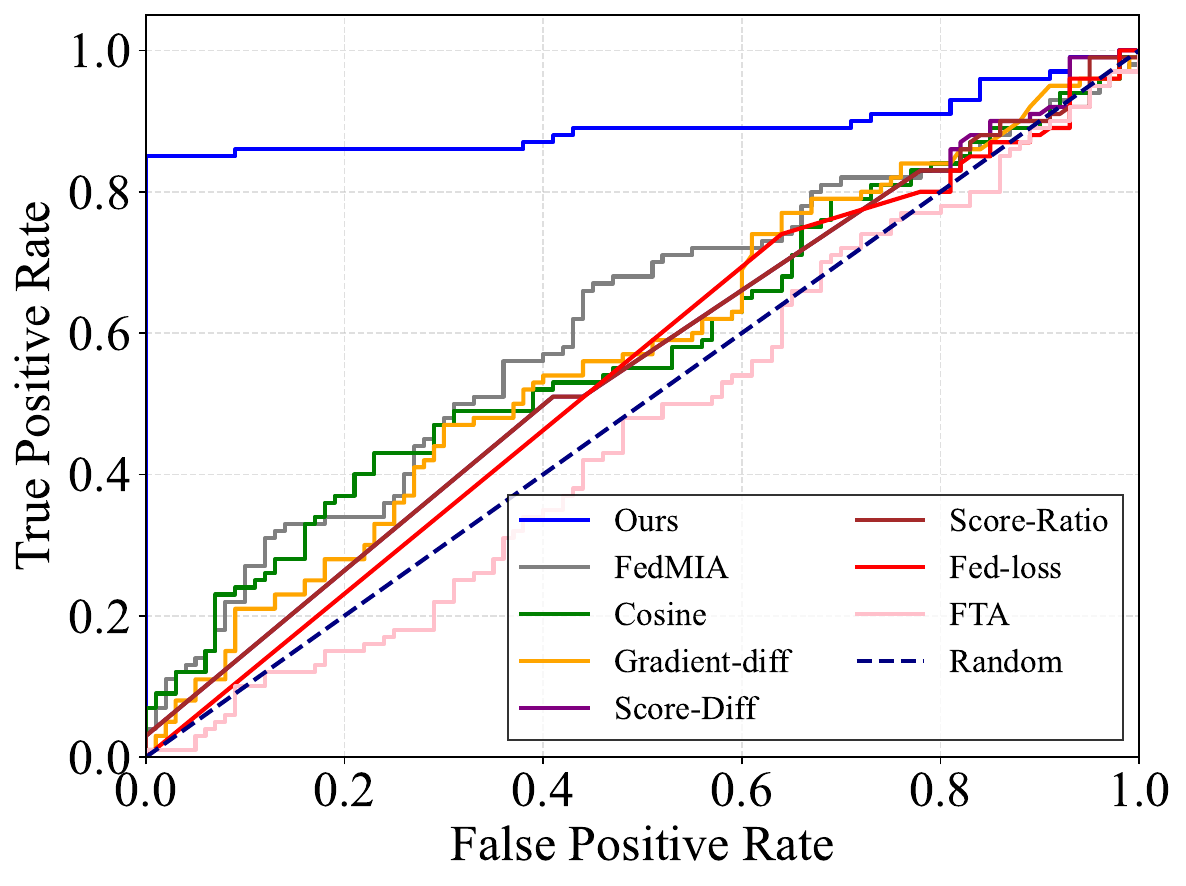} 
			\caption{BERT-Base on Yelp.}
			\label{fig:bert_yelp}
		\end{subfigure}
		\hfill
		\begin{subfigure}[t]{0.23\textwidth}
			\includegraphics[width=\textwidth]{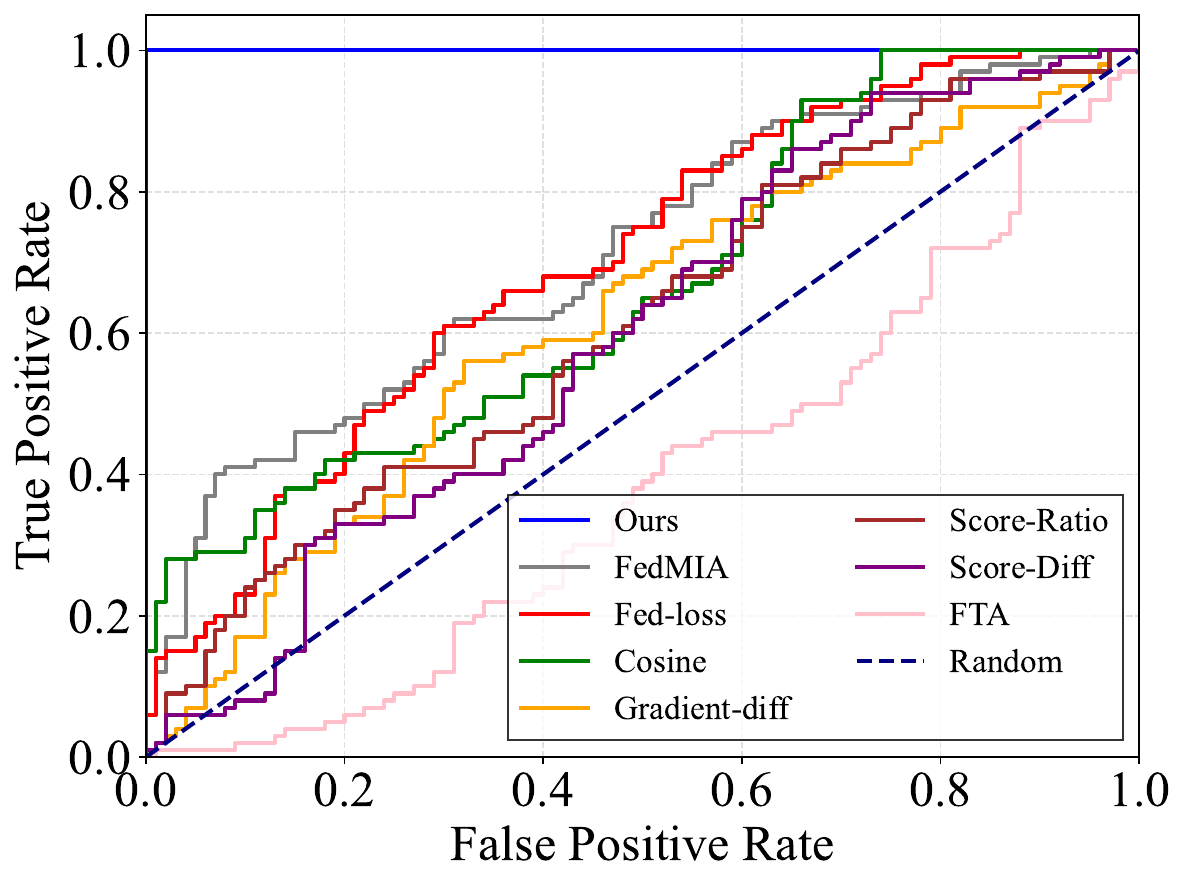}
			\caption{BERT-Base on SST.}
			\label{fig:bert_sst}
		\end{subfigure}
		\hfill
		\begin{subfigure}[t]{0.23\textwidth}
			\includegraphics[width=\textwidth]{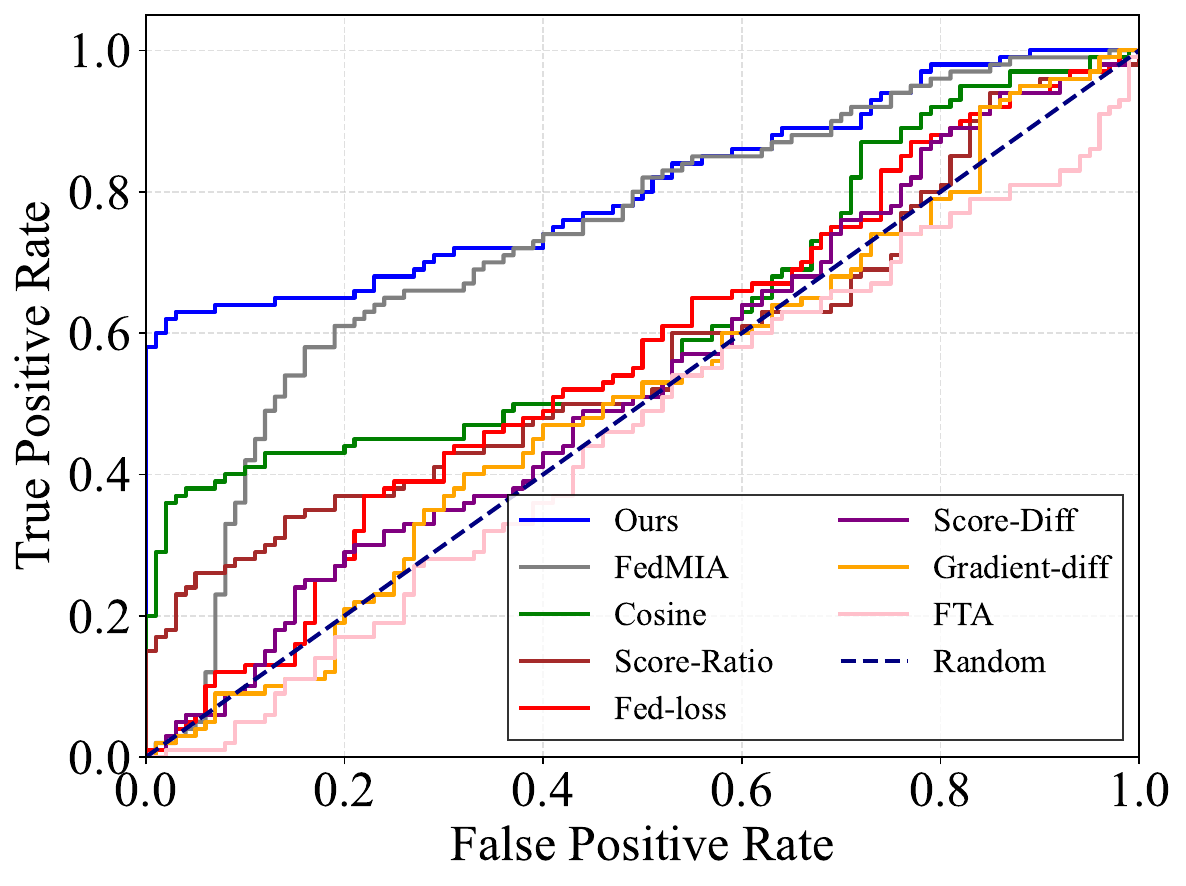}
			\caption{BERT-Base on IMDB.}
			\label{fig:bert_imdb}
		\end{subfigure}
		
		\begin{subfigure}[t]{0.23\textwidth}
			\includegraphics[width=\textwidth]{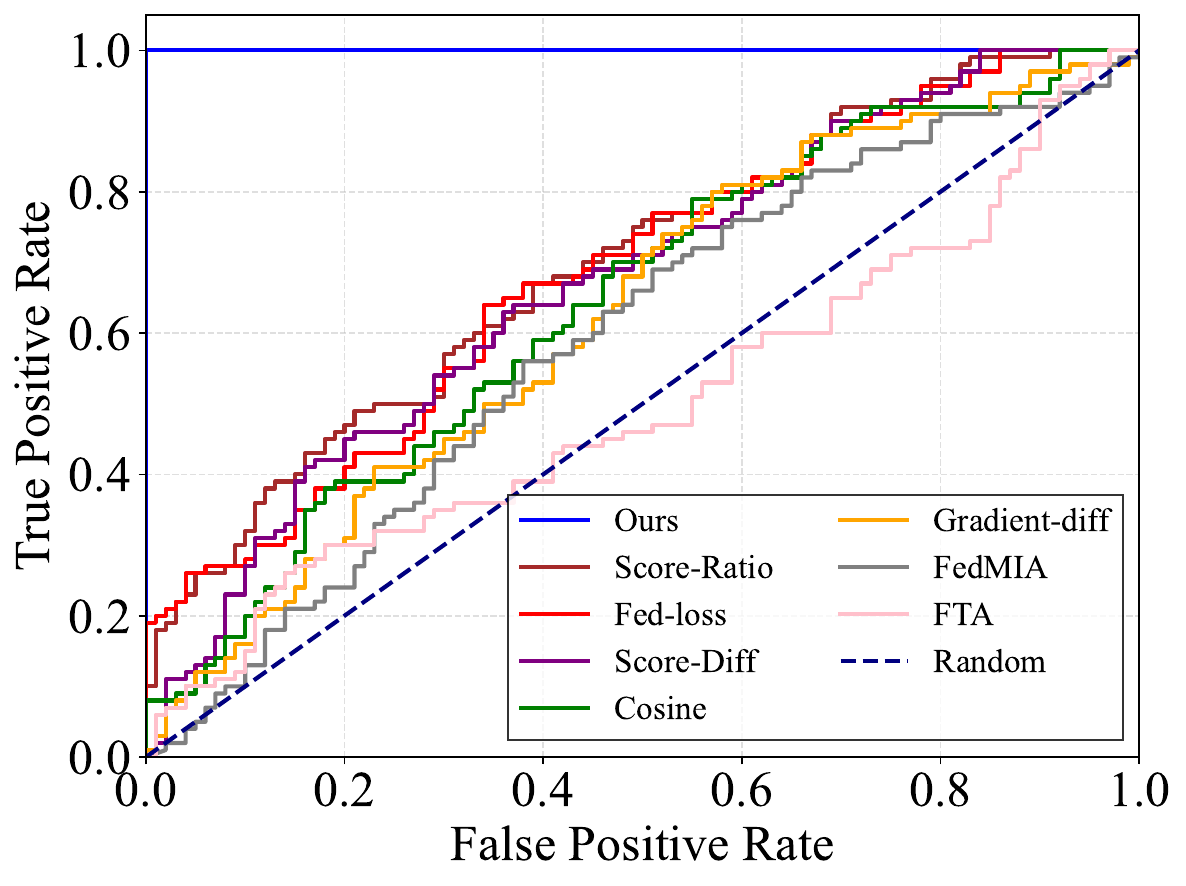}
			\caption{GPT2-Large on CoLA.}
			\label{fig:gpt_cola}
		\end{subfigure}
		\hfill
		\begin{subfigure}[t]{0.23\textwidth}
			\includegraphics[width=\textwidth]{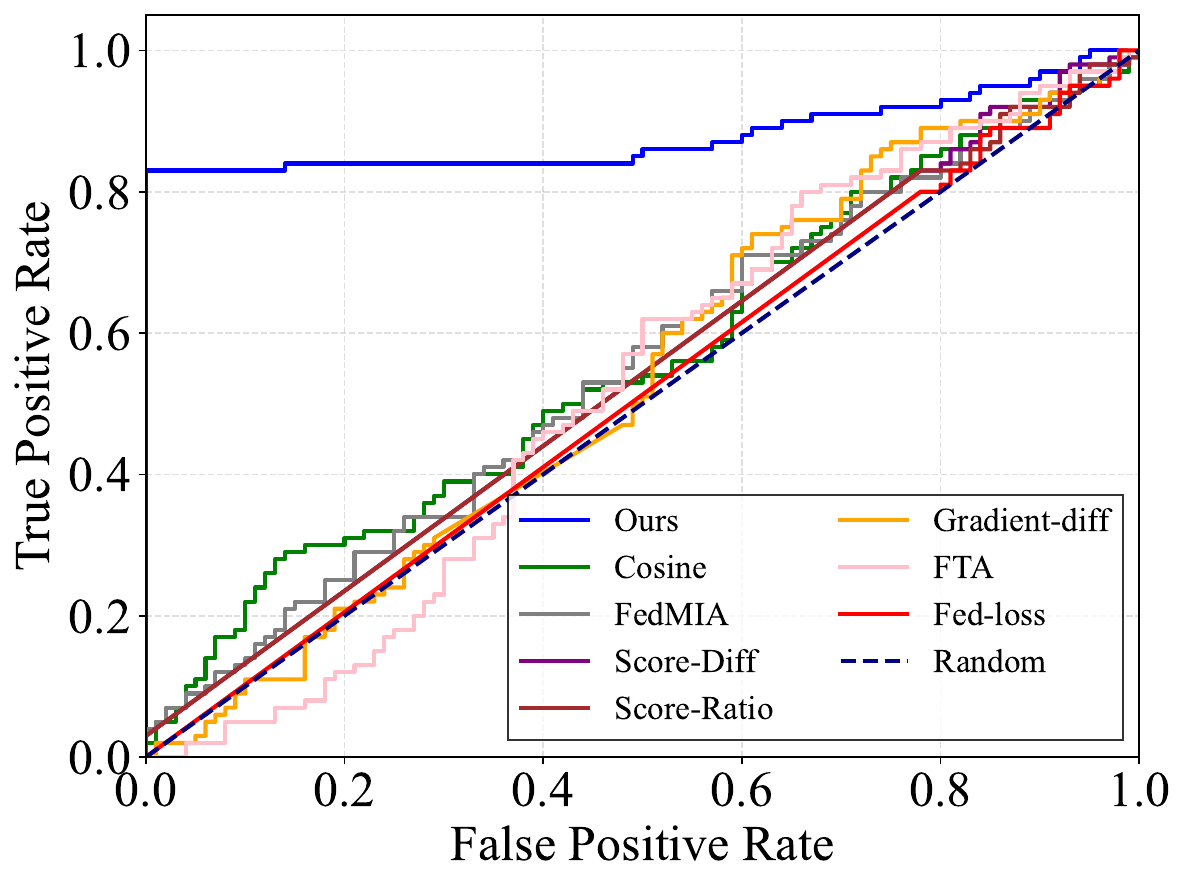} 
			\caption{GPT2-Large on Yelp.}
			\label{fig:gpt_yelp}
		\end{subfigure}
		\hfill
		\begin{subfigure}[t]{0.23\textwidth}
			\includegraphics[width=\textwidth]{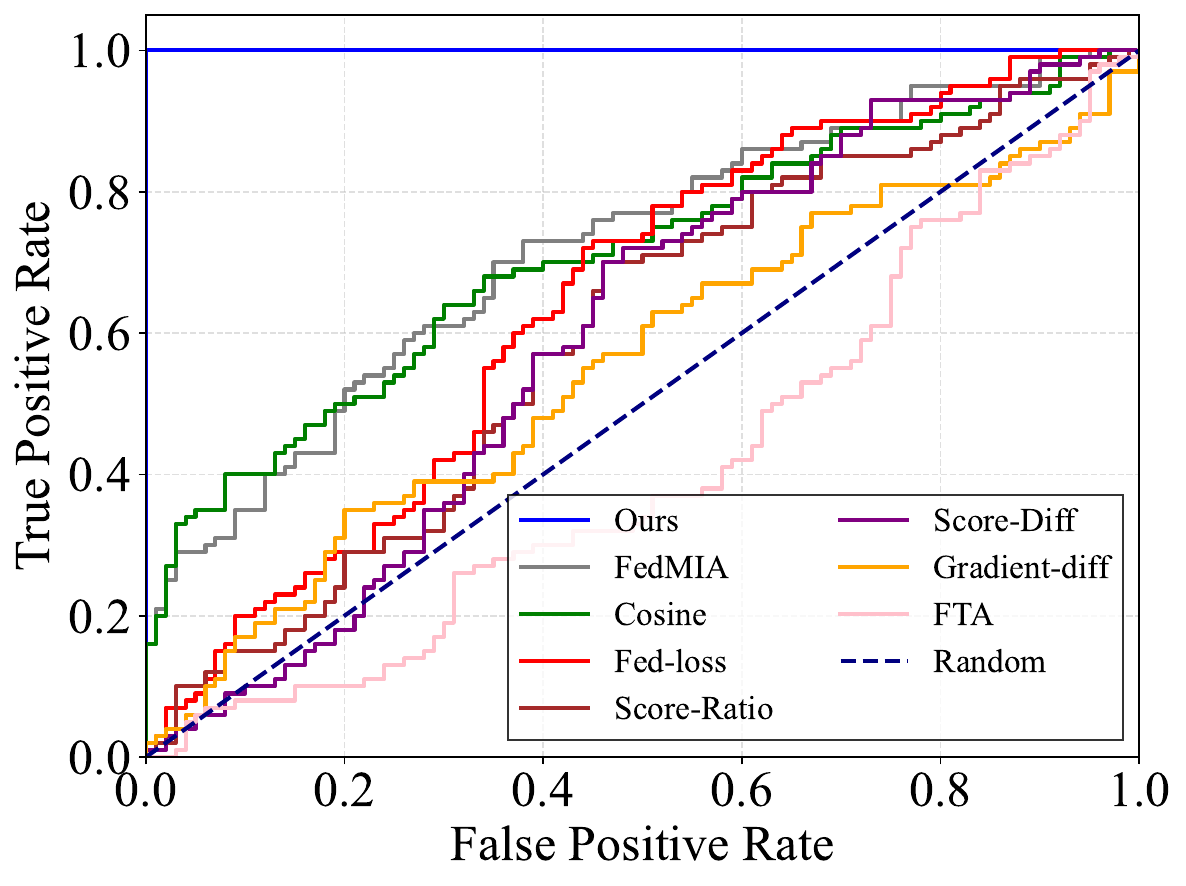}
			\caption{GPT2-Large on SST.}
			\label{fig:gpt_sst}
		\end{subfigure}
		\hfill
		\begin{subfigure}[t]{0.23\textwidth}
			\includegraphics[width=\textwidth]{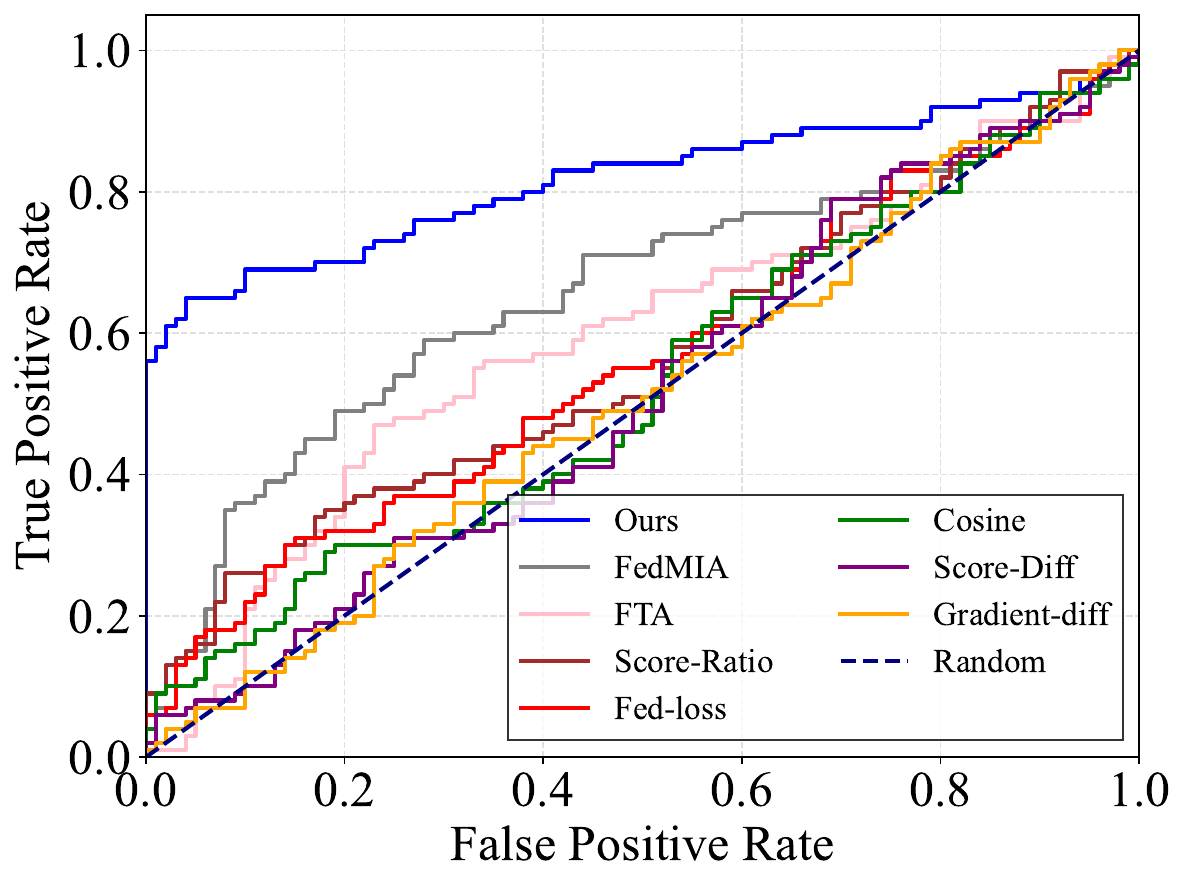}
			\caption{GPT2-Large on IMDB.}
			\label{fig:gpt_imdb}
		\end{subfigure}
		
		\begin{subfigure}[t]{0.23\textwidth}
			\includegraphics[width=\textwidth]{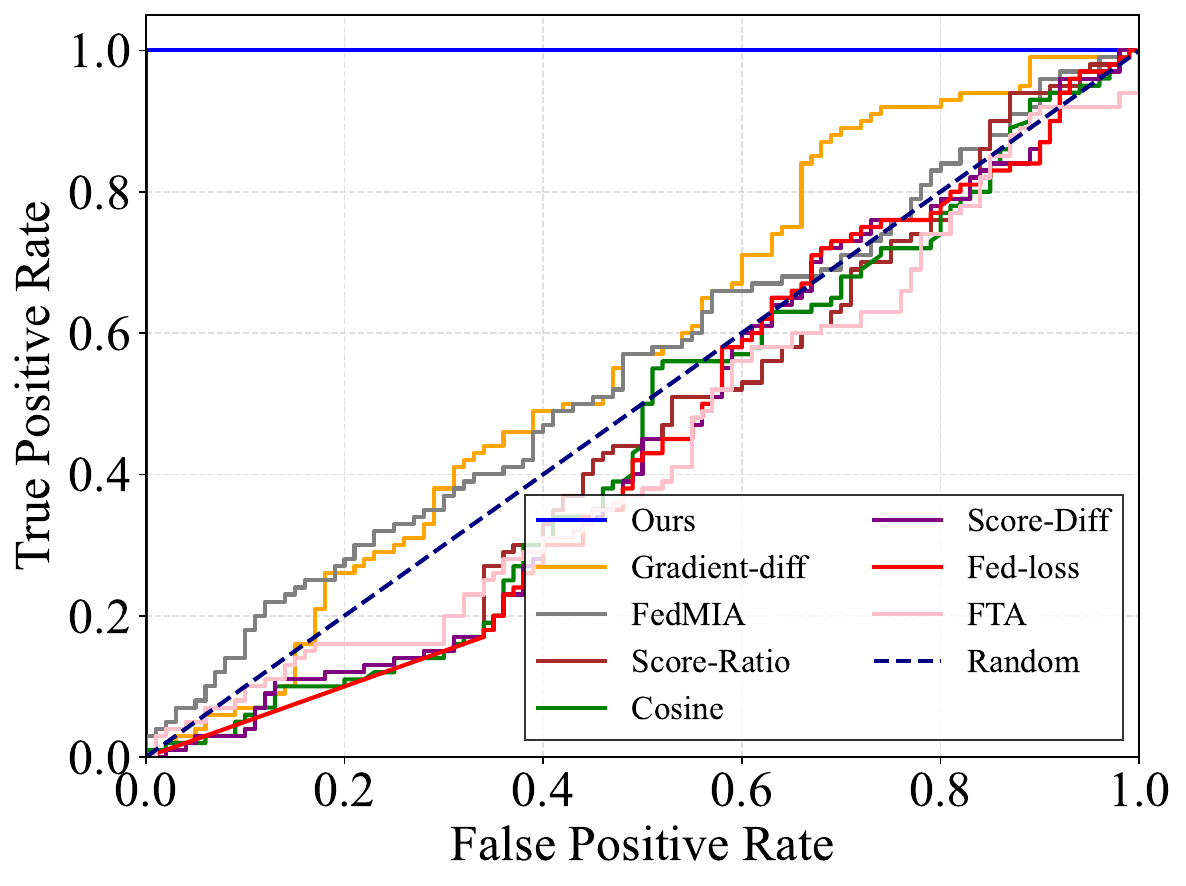}
			\caption{Llama3-8B on CoLA.}
			\label{fig:llama_cola}
		\end{subfigure}
		\hfill
		\begin{subfigure}[t]{0.23\textwidth}
			\includegraphics[width=\textwidth]{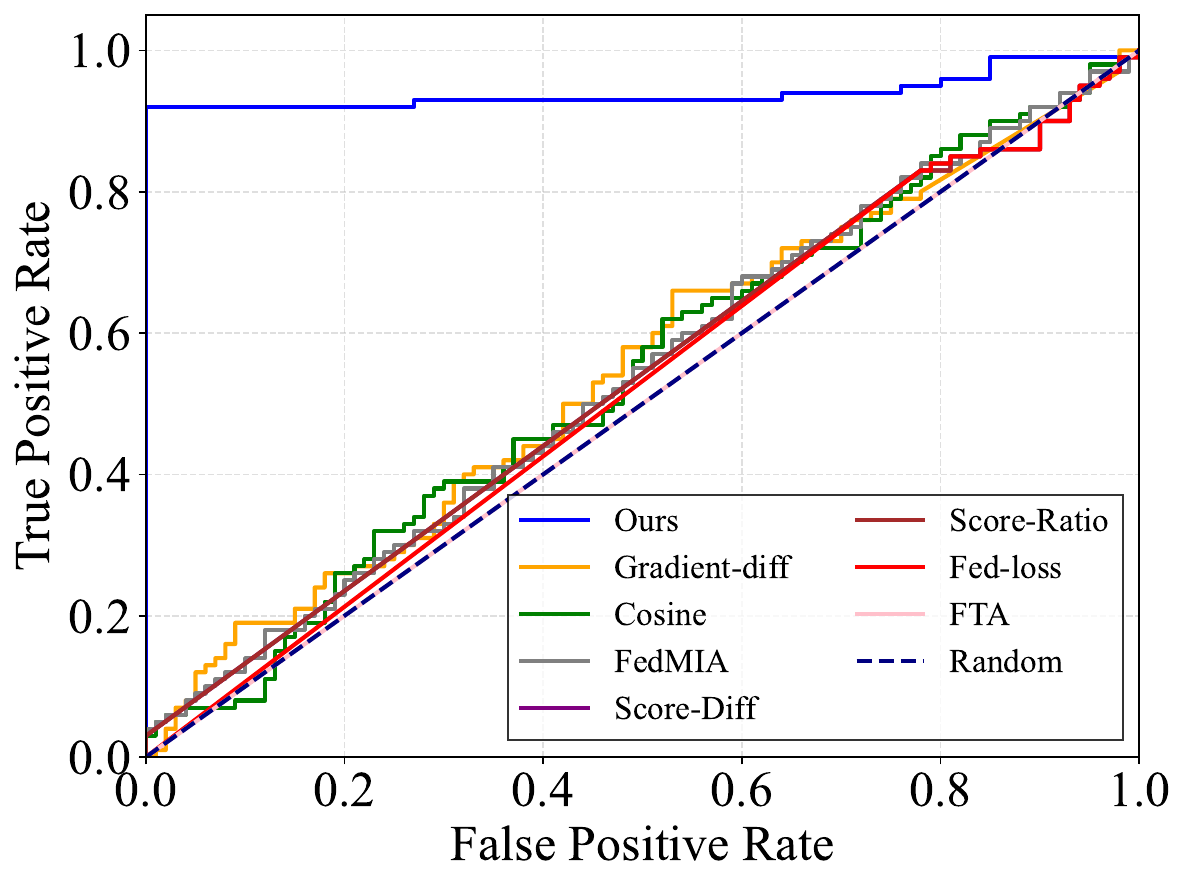} 
			\caption{Llama3-8B on Yelp.}
			\label{fig:llama_yelp}
		\end{subfigure}
		\hfill
		\begin{subfigure}[t]{0.23\textwidth}
			\includegraphics[width=\textwidth]{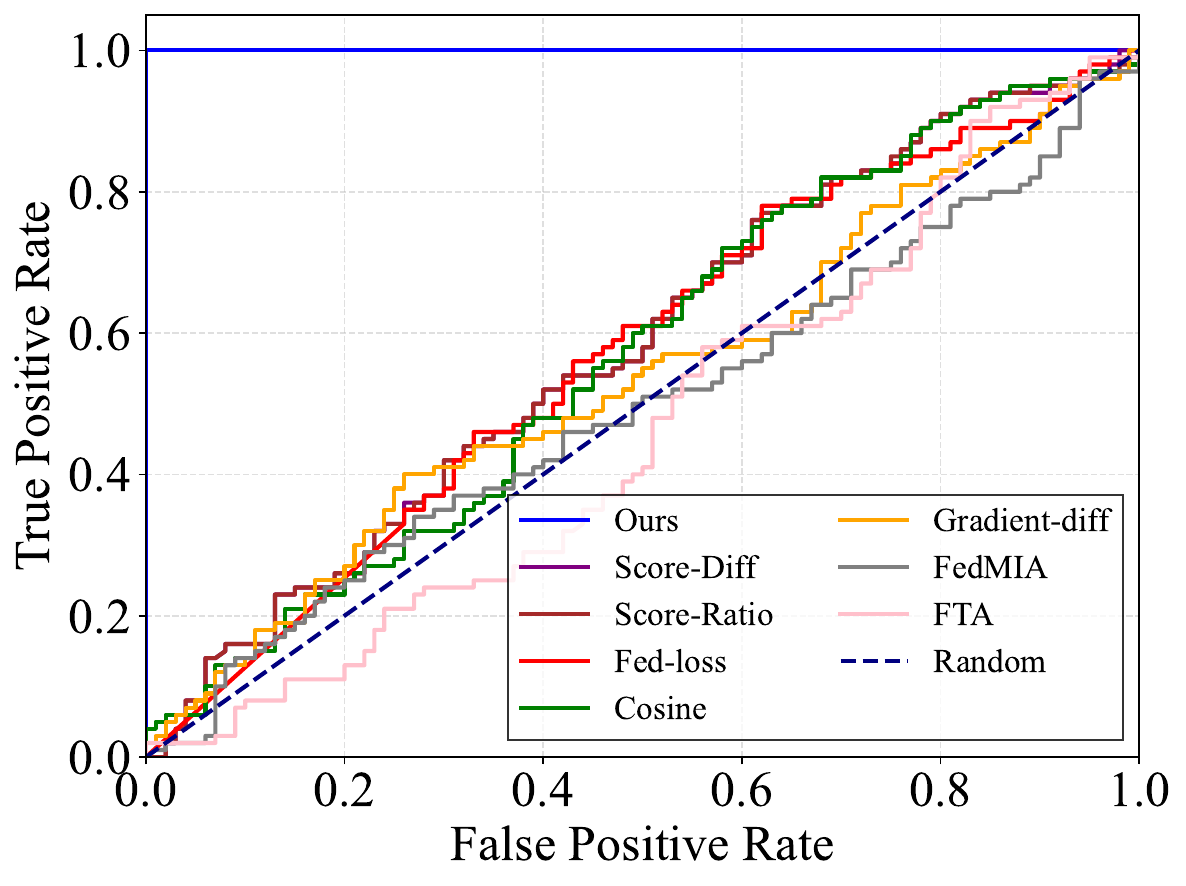}
			\caption{Llama3-8B on SST.}
			\label{fig:llama_sst}
		\end{subfigure}
		\hfill
		\begin{subfigure}[t]{0.23\textwidth}
			\includegraphics[width=\textwidth]{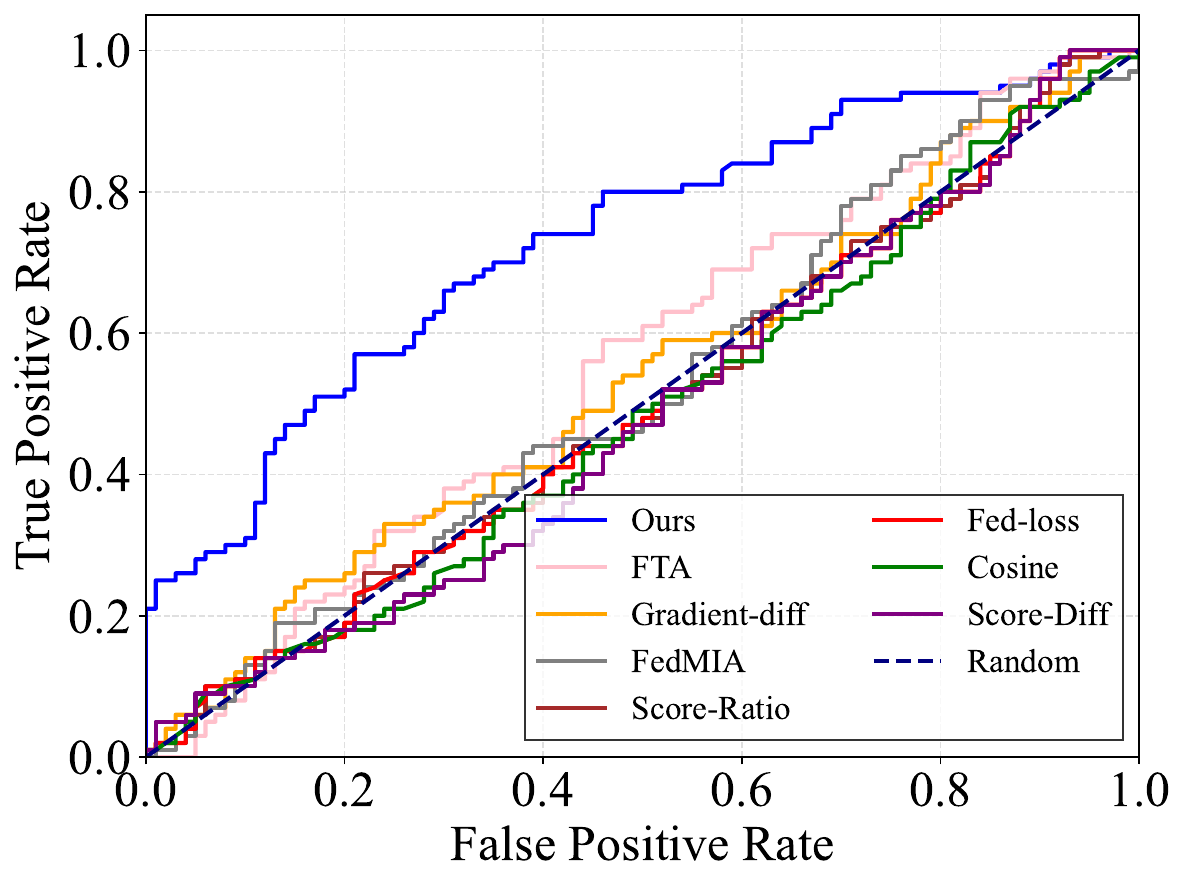}
			\caption{Llama3-8B on IMDB.}
			\label{fig:llama_imdb}
		\end{subfigure}
		
		\begin{subfigure}[t]{0.23\textwidth}
			\includegraphics[width=\textwidth]{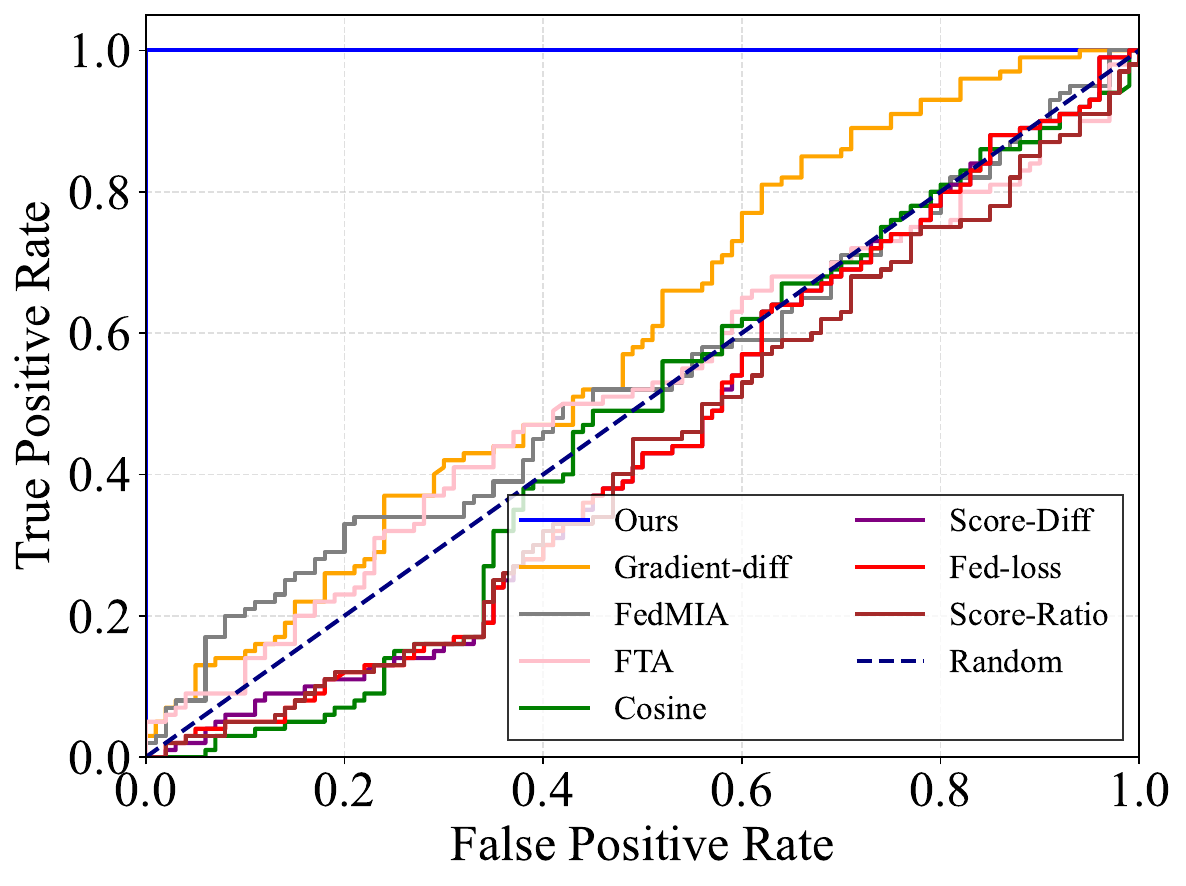}
			\caption{Qwen2.5-14B on CoLA.}
			\label{fig:qwen_cola}
		\end{subfigure}
		\hfill
		\begin{subfigure}[t]{0.23\textwidth}
			\includegraphics[width=\textwidth]{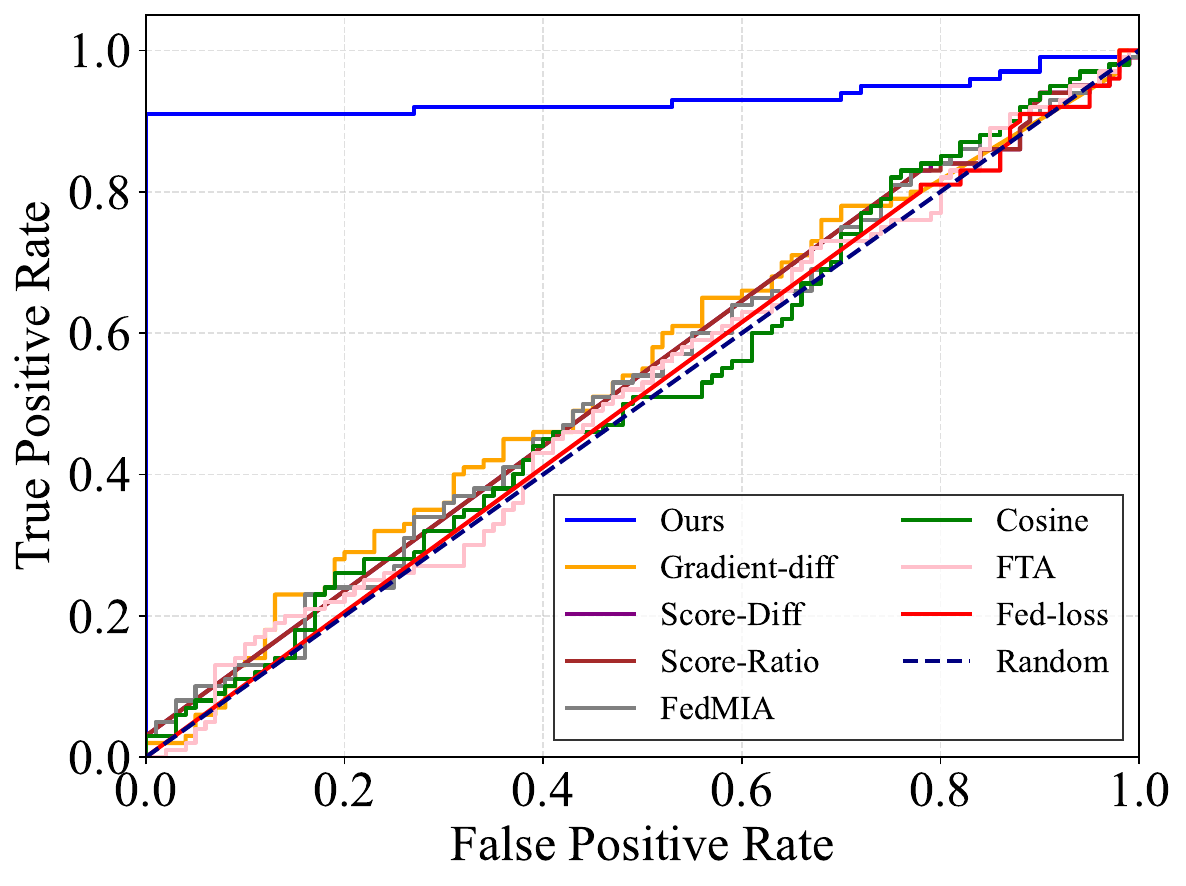} 
			\caption{Qwen2.5-14B on Yelp.}
			\label{fig:qwen_yelp}
		\end{subfigure}
		\hfill
		\begin{subfigure}[t]{0.23\textwidth}
			\includegraphics[width=\textwidth]{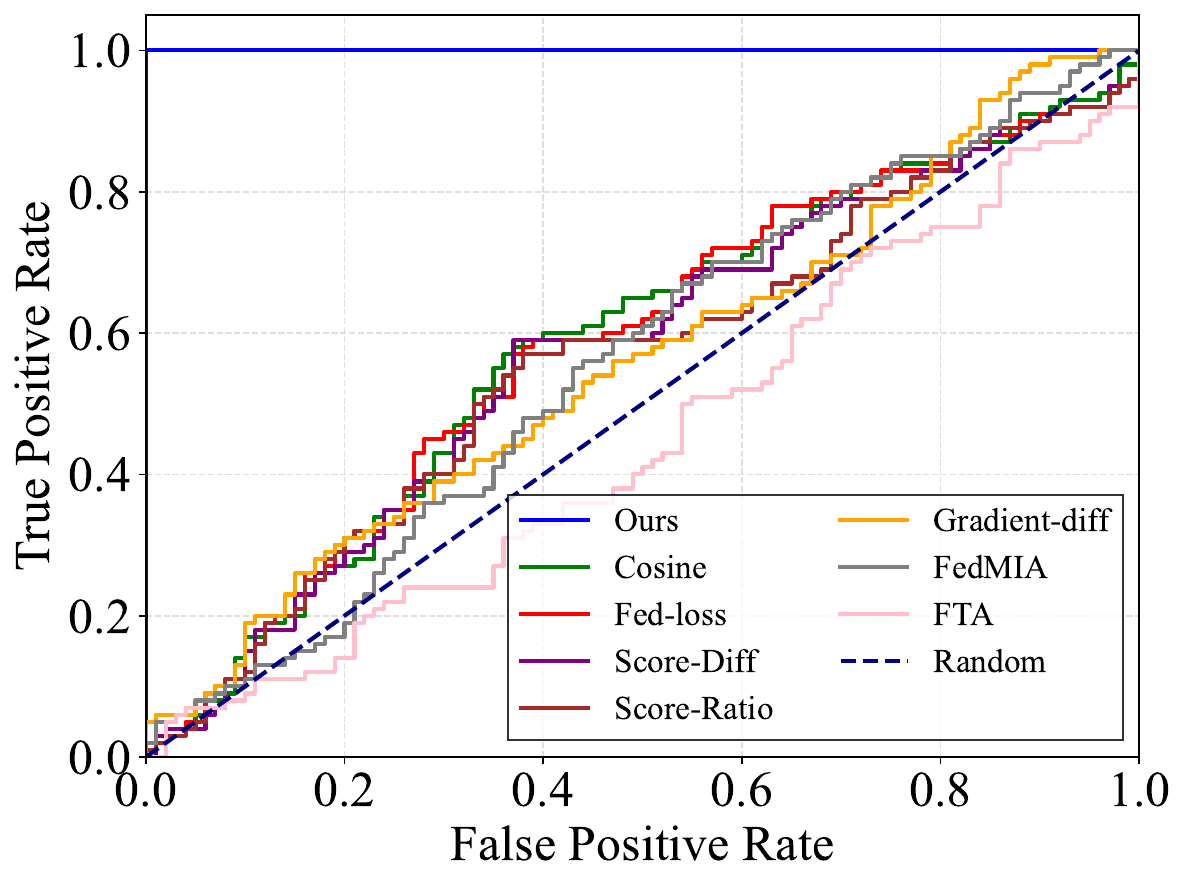}
			\caption{Qwen2.5-14B on SST.}
			\label{fig:qwen_sst}
		\end{subfigure}
		\hfill
		\begin{subfigure}[t]{0.23\textwidth}
			\includegraphics[width=\textwidth]{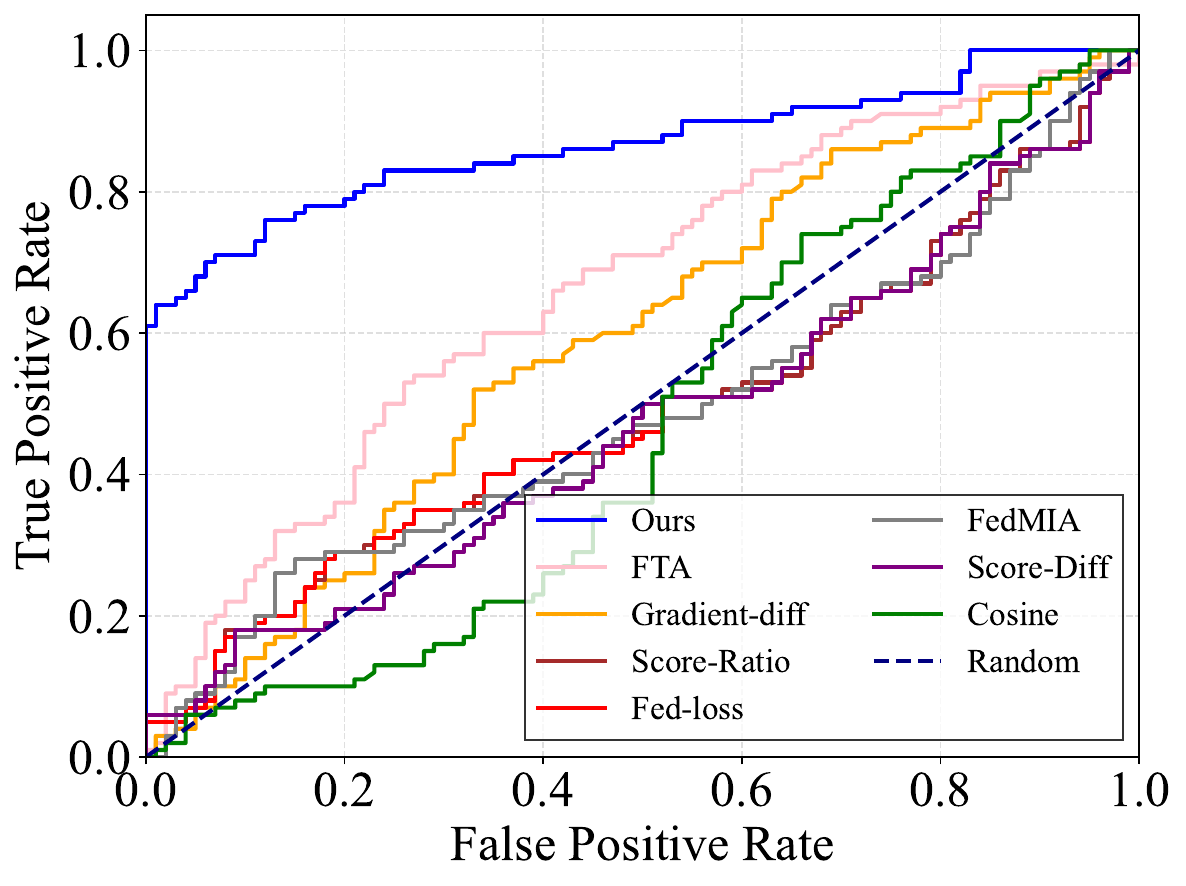}
			\caption{Qwen2.5-14B on IMDB.}
			\label{fig:qwen_imdb}
		\end{subfigure}
		
		\caption{Comparison of attack performance across different models and datasets using ROC curves.}
		\label{fig:roc_comparison}
		\vspace{-0.3cm}
	\end{figure*}
	
	\para{FedLLM Setup.} Unless otherwise specified, we assume a system with 30 clients that collaboratively train a global model in a synchronous manner under the coordination of a central server. The training data is uniformly and randomly partitioned among the clients. Attack performance is typically evaluated at the 50-th communication round.

	\para{Baselines.} We evaluate our attack against the seven most representative MIAs introduced in \S\ref{Baseline} as competing baselines, all under identical experimental conditions. The seven MIAs are FedLoss~\cite{ICLR_MIA}, Cosine~\cite{ICLR_MIA}, Gradient-Diff~\cite{ICLR_MIA}, Score-Diff~\cite{ref_MIA2}, Score-Ratio~\cite{ref_MIA2}, FTA~\cite{FTA}, and FedMIA~\cite{FedMIA}.
	
	\para{Evaluation Metrics.} In each evaluation trial, we perform membership inference on one positive sample (\ie a sample in the current training batch) and one negative sample (\ie a non-member sample). Each evaluation is repeated 100 times, ensuring that the total number of positive predictions satisfies \(T_p + F_p = 100\) and the total number of negative predictions satisfies \(T_n + F_n = 100\). The attack performance is comprehensively assessed using the Receiver Operating Characteristic (ROC) curve and its corresponding Area Under the ROC Curve (AUC) metric.

	\subsection{Comparison of Baselines}\label{main}
	We evaluate \texttt{ProjRes} against seven baseline membership inference attacks (MIAs) across four datasets and four LLMs, using Adapter-based PEFT for fine-tuning. Following \cite{adapter_structure}, a trainable Adapter module (down-projection ratio 2) is inserted after each Transformer layer, with all other model parameters frozen. Each client uses a local batch size of 16, a common setting in FedLLM training \cite{Batch_size1, Batch_size2}. The resulting ROC curves, shown in Fig.~\ref{fig:roc_comparison}, demonstrate that \texttt{ProjRes} consistently outperforms all baselines, with curves closely approaching the top-left corner of the ROC space, indicating superior discrimination between member and non-member samples.	These results are further confirmed by the AUC scores in Table~\ref{acu_score}: \texttt{ProjRes} achieves the highest AUC in all settings, often reaching or nearing 1.0, and significantly surpassing existing methods. Relative improvements over the best baseline range from 9.65\% to 75.75\%, highlighting its substantial advantage. As discussed in \S\ref{Baseline}, this stems from the misalignment of conventional MIA designs with the architectural and training dynamics of FedLLMs, whereas \texttt{ProjRes} is specifically tailored to exploit gradient semantics in such settings.

	\begin{table*}[!t]
	\centering
	\renewcommand{\arraystretch}{1.2}
	\setlength{\tabcolsep}{1.9pt}
	\caption{AUC scores corresponding to the ROC curves in Fig.~\ref{fig:roc_comparison}. Higher AUC indicates better attack performance. ``Relative Gain'' shows the percentage improvement of \texttt{ProjRes} over the \underline{best baseline} against each model on each dataset.}
	\label{tab:ours_vs_best}
	\begin{tabular}{l cccc cccc cccc cccc}
		\hline
		\multicolumn{1}{l}{\multirow{2}{*}{\textbf{Method}}} & \multicolumn{4}{c}{\textbf{BERT-Base}} & \multicolumn{4}{c}{\textbf{GPT2-Large}} & \multicolumn{4}{c}{\textbf{Llama3-8B}} & \multicolumn{4}{c}{\textbf{Qwen2.5-14B}} \\
		\cmidrule(lr){2-5} \cmidrule(lr){6-9} \cmidrule(lr){10-13} \cmidrule(lr){14-17}
		& CoLA    & Yelp    & SST     & IMDB     & CoLA    & Yelp    & SST     & IMDB      & CoLA    & Yelp    & SST     & IMDB      & CoLA   & Yelp     & SST     & IMDB \\ \hline
		FedMIA        & \underline{0.748}   & \underline{0.611}   & \underline{0.713}   & \underline{0.736}    & 0.591   & 0.549   & \underline{0.720}   & \underline{0.657}     & 0.545   & 0.537   & 0.494   & 0.521     & 0.530   & 0.530   & 0.557   & 0.487 \\ 
		Cosine        & 0.745   & 0.582   & 0.665   & 0.627    & 0.641   & \underline{0.558}   & 0.709   & 0.526     & 0.450   & 0.539   & 0.567   & 0.484     & 0.464   & 0.518   & \underline{0.584}   & 0.466 \\ 
		Fed-loss      & 0.685   & 0.541   & 0.700   & 0.561    & 0.684   & 0.509   & 0.643   & 0.559     & 0.444   & 0.520   & 0.570   & 0.497     & 0.444   & 0.508   & 0.582   & 0.490 \\ 
		Score-Ratio   & 0.674   & 0.553   & 0.612   & 0.568    & \underline{0.696}   & 0.535   & 0.595   & 0.567     & 0.451   & 0.532   & 0.579   & 0.498     & 0.428   & 0.534   & 0.558   & 0.490 \\ 
		Score-Diff    & 0.636   & 0.556   & 0.596   & 0.532    & 0.671   & 0.539   & 0.594   & 0.513     & 0.447   & 0.532   & \underline{0.580}   & 0.484     & 0.445   & 0.534   & 0.570   & 0.471 \\ 
		Gradient-diff & 0.565   & 0.574   & 0.616   & 0.510    & 0.624   & 0.531   & 0.553   & 0.509     & \underline{0.569}   & \underline{0.546}   & 0.536   & 0.533     & \underline{0.593}   & \underline{0.543}   & 0.557   & 0.592 \\ 
		FTA           & 0.490   & 0.446   & 0.385   & 0.459    & 0.500   & 0.521   & 0.412   & 0.592     & 0.438   & 0.500   & 0.462   & \underline{0.553}     & 0.518   & 0.515   & 0.449   & \underline{0.662} \\   
		\textbf{Ours}       & \textbf{1.000}   & \textbf{0.894}   & \textbf{1.000}   & \textbf{0.807}    & \textbf{1.000}   & \textbf{0.880}   & \textbf{1.000}   & \textbf{0.819}     & \textbf{1.000}   & \textbf{0.940}   & \textbf{1.000}   & \textbf{0.731}     & \textbf{1.000}   & \textbf{0.933}   & \textbf{1.000}   & \textbf{0.867} \\  \hline
		Relative Gain & 33.68\% & 46.32\% & 40.25\% & \textbf{9.65\%}   & 43.68\% & 36.60\% & 38.89\% & 24.66\%   & \textbf{75.75\%} & 72.16\% & 72.41\% & 32.19\%   & 68.63\% & 71.82\% & 71.23\% & 30.97\% \\  \hline
	\end{tabular}
	\label{acu_score}
	\vspace{-0.2cm}
\end{table*}

	\textit{Experimental Analysis.} Next, we provide a detailed analysis of \texttt{ProjRes}’s experimental results. In Fig. \ref{fig:roc_comparison}, each row of subplots presents evaluation results across different datasets, \ie CoLA, Yelp, SST, and IMDB (from left to right), using the same LLM. These datasets have progressively increasing average text lengths, resulting in a higher number of tokens per batch under the same batch size. Each column of subplots shows results for the same dataset across different LLMs, \ie BERT-Base, GPT2-Large, Llama3-8B, and Qwen2.5-14B, arranged from top to bottom in order of increasing model scale. Accordingly, the dimensionality of the hidden embeddings increases with model size.
	
	Further examination of Fig. \ref{fig:roc_comparison} and Table \ref{acu_score} reveals that, for a fixed model, the ROC curves gradually degrade and the AUC decreases as input sequence length increases (from CoLA to IMDB). Conversely, for datasets like Yelp and IMDB, the attack performance of \texttt{ProjRes} improves as the model’s embedding dimension grows. This aligns with the analysis in Appendix \ref{Basic_Capability}: larger embedding dimensions and a greater number of neurons in the trainable module amplify the privacy leakage signals exploited by \texttt{ProjRes}, enhancing its effectiveness.
	
	Interestingly, the results indicate that \texttt{ProjRes} achieves AUC scores of 1.0 on both CoLA and SST. However, the average sequence length of Yelp lies between CoLA and SST, its AUC falls slightly below 1.0. This discrepancy arises from the padding mechanism used in LLMs during batch processing: sequences are padded to the length of the longest sequence in a batch. Yelp exhibits higher length variance than SST, leading to more tokens processed per batch, effectively simulating longer input sequences and slightly reducing attack performance. In the next section, we extend this investigation to evaluate \texttt{ProjRes} across FedLLMs using different PEFT strategies.


	\subsection{Impact of FedLLM Fine-Tuning Strategy}
	\label{Diverse_Implementation}
	
	In this section, we evaluate \texttt{ProjRes} across five FedLLM configurations with different trainable modules: 1) Adapter (down-projection ratio 1.2); 2) Adapter (ratio 2); 3) LoRA (ratio 2, applied in parallel with Qproj); 4) Qproj; and 5) DownProj. These modules vary primarily in their neuron count and placement within the Transformer layers. Table~\ref{trainable_modules} summarizes their sizes for Qwen2.5-14B.
	Experiments are conducted on SST with increasing batch sizes. Due to the dataset's moderate and consistent sequence lengths, the total number of tokens per batch, denoted as $p$ in \S\ref{sec-3-4}, rises steadily with batch size, enabling clearer observation of attack performance trends. Results in Table~\ref{AUC_batch_size} show that \texttt{ProjRes} is effective across all configurations and reveal a clear correlation: larger trainable modules lead to higher privacy leakage risk, indicating that model capacity in the exposed components directly impacts vulnerability.

	\begin{table}[!t]
		\renewcommand{\arraystretch}{1.2}
		\setlength{\tabcolsep}{4pt}
		\centering
		\caption{The number of neurons in different trainable modules in Qwen2.5-14B.}
		\label{trainable_modules}
		\begin{threeparttable}
			{\begin{tabular}{cccccc}\toprule
					\makecell{\textbf{Trainabel}\\\textbf{Module}}   & DownProj    & Qproj   & Adapter1.2   & Adapter2  & LoRA2            \\\midrule
					\makecell{\textbf{\# of}\\\textbf{Neurons}}    &  5120       & 5120    & 4267         & 2560      & 2560             \\              \bottomrule
			\end{tabular}}    
		\end{threeparttable}
		\vspace{-0.3cm}
	\end{table}

	\begin{table*}[!t]
		\renewcommand{\arraystretch}{1.2}
		\setlength{\tabcolsep}{4pt}
		\centering
		\caption{AUC of \texttt{ProjRes} against Qwen2.5 with different trainable module configurations under varying batch sizes.}
		\label{AUC_batch_size}
		\begin{threeparttable}
			\begin{tabular}{l *{15}{c}} \hline
				\multirow{2}{*}{\textbf{Trainable Module}} & \multicolumn{15}{c}{\textbf{Batch Size}} \\ \cmidrule(lr){2-16}
				& 1     & 2     & 4     & 8     & 16    & 32    & 64    & 128    & 256  & 384   & 512   & 640   & 786   & 896   & 1024 \\ \hline
				DownProj   & 1.000 & 1.000 & 1.000 & 1.000 & 1.000 & 1.000 & 1.000 & 0.995 & 0.992 & 0.967 & 0.970 & 0.967 & 0.970 & 0.963 & 0.963 \\
				Qproj      & 1.000 & 1.000 & 1.000 & 1.000 & 1.000 & 1.000 & 1.000 & 0.993 & 0.987 & 0.953 & 0.934 & 0.953 & 0.934 & 0.943 & 0.953 \\
				Adapter1.2 & 1.000 & 1.000 & 1.000 & 1.000 & 1.000 & 1.000 & 1.000 & 0.996 & 0.972 & 0.980 & 0.801 & 0.980 & 0.801 & 0.946 & 0.744 \\
				Adapter2   & 1.000 & 1.000 & 1.000 & 1.000 & 1.000 & 1.000 & 1.000 & 0.983 & 0.947 & 0.711 & 0.768 & 0.711 & 0.768 & 0.696 & 0.637 \\
				LoRA2      & 1.000 & 1.000 & 1.000 & 1.000 & 1.000 & 1.000 & 1.000 & 0.998 & 0.988 & 0.901 & 0.957 & 0.901 & 0.957 & 0.949 & 0.958 \\ \hline
			\end{tabular}
		\end{threeparttable}
		\vspace{-0.5cm}
	\end{table*}
	


	Furthermore, we observe that as the batch size increases, the attack performance of \texttt{ProjRes} against Qwen2.5-14B equipped with Adapter2 degrades most noticeably. Specifically, at a batch size of 1024, its AUC score is 0.32 lower than that of Qwen2.5-14B with LoRA2; even Adapter1.2, which contains more neurons, exhibits an AUC that is 0.22 lower. This behavior arises from the activation functions in the Adapter architecture: when certain activations remain inactive, the corresponding neurons do not contribute, effectively reducing the number of usable neurons. Importantly, the activation status of neurons evolves dynamically during training. To further investigate this phenomenon, we systematically evaluate \texttt{ProjRes}’s attack performance at different training stages in the next section.
	
	\subsection{Imapact of Training Stages}
	\label{sub_AUC_epoch}
	
	To assess \texttt{ProjRes}'s attack performance across different stages of FedLLM training, we fine-tune BERT-Base with a trainable Adapter (reduction factor of 2) on CoLA as the downstream task. Fine-tuning BERT on CoLA is both a practical task choice and, as indicated by Fig. \ref{fig:roc_comparison}, a setting where baseline methods exhibit stronger attack performance, enabling a more meaningful comparison. During federated training, each client uses a local batch size of 16 and a learning rate of $1 \times 10^{-4}$. MIAs are launched after training epochs \{1, 10, 20, 40, 100, 200, 400\}, covering the early, middle, and late stages of training. The corresponding evaluation results are presented in Fig. \ref{AUC_epoch}.
	
	Fig. \ref{AUC_epoch} shows that as training progresses, the AUC scores of all baseline methods generally decline, whereas \texttt{ProjRes} consistently achieves perfect attack performance (AUC = 1.0) across all training stages, markedly outperforming every baseline. These results indicate that \texttt{ProjRes}'s attack performance remains stable and does not degrade as model training advances. This behavior stems from a fundamental difference in attack design: existing baseline methods rely on dynamic signals closely tied to the training process, for example, Fed-loss depends on the model's output loss, whose discriminative power diminishes as the model converges. In contrast, \texttt{ProjRes} leverages a strict algebraic relationship between gradients and sample embedding vectors, which remains invariant throughout training. This ensures both robustness and superior attack efficacy across all stages of model fine-tuning. More results are provided in Appendix \ref{c1} to further corroborating the findings.
	
	\begin{figure}[!t]
		\centering
		\includegraphics[width=0.41\textwidth]{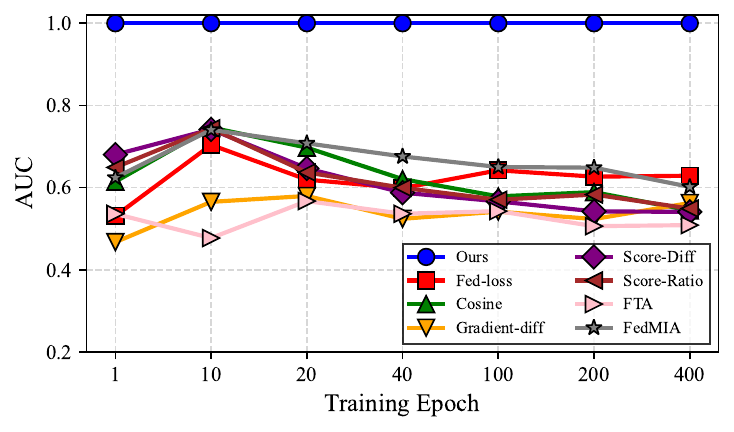} %
		\caption{AUC scores of MIAs against Adapter-based BERT-base across different training epochs on CoLA.}
		\label{AUC_epoch}
		\vspace{-0.3cm}
	\end{figure} 
	
	\subsection{Impact of Trainable Layer Positions}
	\label{layer_auc}
	We aim to investigate the impact of trainable layer positions on \texttt{ProjRes}. The experimental setup follows that described in  \S\ref{sub_AUC_epoch}, with the sole difference being the insertion strategy of the Adapter within BERT-Base. In this experiment, a single Adapter module is embedded into different Transformer layers of BERT-Base to create 12 distinct trainable configurations. As illustrated in Fig. \ref{AUC_layer}, \texttt{ProjRes} consistently achieves a perfect AUC of 1.0 across all 12 Adapter insertion positions, indicating that it can perform precisely accurate membership inference regardless of where the trainable module is located. This finding highlights the remarkable robustness and generalization capability of the proposed method. The key reason lies in \texttt{ProjRes}'s design: it is grounded in a strict algebraic relationship between gradient updates and the hidden embeddings of training samples. This relationship remains largely invariant to the specific layer where the trainable module is inserted, ensuring stable attack performance across different architectural configurations. Additional evidence supporting the conclusion is provided in Appendix \S\ref{c2}.

	
	\begin{figure}[!t]
		\centering
		\includegraphics[width=0.41\textwidth]{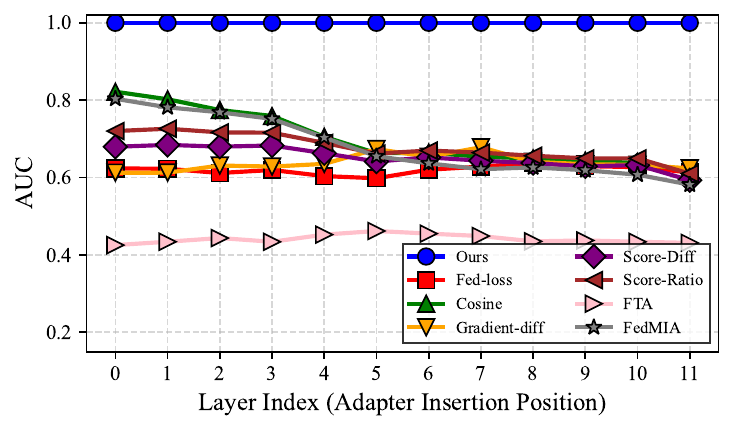} %
		\caption{The AUC scores of \texttt{ProjRes} against BERT-Base with a single adapter inserted at different Transformer layers on CoLA.}
		\label{AUC_layer}
		\vspace{-0.3cm}
	\end{figure}
	
		    \begin{figure}[!t]
		\centering
		\includegraphics[width=0.41\textwidth]{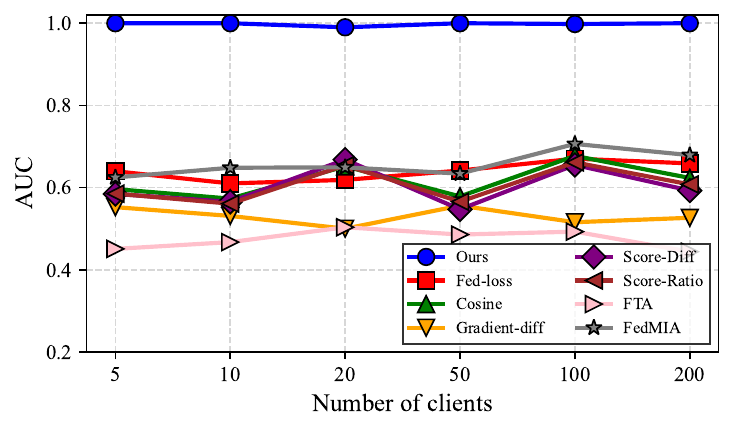}
		\caption{AUC scores of MIAs against Adapter-based BERT-base on CoLA under varying numbers of clients.}
		\label{AUC_client}
		\vspace{-0.3cm}
	\end{figure} 
	
	
	\subsection{Cost Evaluation}
	We compare \texttt{ProjRes} with baseline methods in terms of memory consumption and runtime overhead to evaluate attack efficiency. Table \ref{memory_overhead} summarizes the memory requirements of each approach. For clarity, the memory cost of model loading is reported separately in the ``Model Only'' row, while the remaining rows indicate the additional memory required to execute each attack. Due to hardware limitations, full model architectures were used only for BERT-Base and GPT2-Large, whereas LLaMA3-8B and Qwen2.5-14B were evaluated using only their first two Transformer layers. As shown, although \texttt{ProjRes} does not have the lowest memory footprint among all methods, its additional memory usage is modest and remains well within an acceptable range. Table \ref{time_overhead} presents the time overhead of various MIAs across the four models. The results show that \texttt{ProjRes} incurs noticeably higher execution times on BERT-Base and GPT2-Large compared to other approaches. This is primarily because \texttt{ProjRes} performs attacks independently on each Transformer layer (as detailed in \S\ref{layer_auc}). For BERT-Base and GPT2-Large, which utilize all Transformer layers, this results in multiple iterative attacks, leading to a significant increase in computational overhead.

	\begin{table}[!t]
		\renewcommand{\arraystretch}{1.2}
		\setlength{\tabcolsep}{4pt}
		\centering
		\caption{Memory consumption (in MB) of various MIA against the four models on CoLA. The ``Model Only'' row denotes the memory overhead of loading the model, while the other rows represent the additional memory overhead incurred by the respective methods.}
		\label{memory_overhead}
		\begin{threeparttable}
			{\begin{tabular}{lcccc}\toprule
					\textbf{Method}      & \textbf{BERT-Base} & \textbf{GPT2-Large}   & \textbf{Llama3-8B}   & \textbf{Qwen2.5-14B}     \\\midrule
					Model Load   &  461.54              & 3286.59                 &  5817.10               & 5526.25     \\
					FedMIA       &  109.148             & 902.509                 &  510.231               & 800.157                  \\
					Cosine       &  82.097              & 677.982                 &  382.186               & 600.100                \\
					Fed-loss     &  0.004               & 0.003                   &  0.003                 & 0.003                \\ 
					Score-Ratio  &  0.008               & 0.007                   &  0.007                 & 0.007                \\ 
					Score-Diff   &  0.008               & 0.007                   &  0.007                 & 0.007                \\ 
					Gradient-Diff&  82.026              & 677.98                  &  382.187               & 600.101                \\ 
					FTA          &  0.007               & 0.006                   &  0.006                 & 0.006                \\
					Ours      &  0.014               & 0.013                   &  0.033                 & 0.041              \\\bottomrule
			\end{tabular}}       
		\end{threeparttable}
		\vspace{-0.2cm}
	\end{table}

		\begin{table}[!t]
		\renewcommand{\arraystretch}{1.2}
		\setlength{\tabcolsep}{4pt}
		\centering
		\caption{Time overhead (in seconds) of various MIAs on CoLA across four models.}
		\label{time_overhead}
		\begin{threeparttable}
			{\begin{tabular}{lcccc}\toprule
					\textbf{Method}     &  \textbf{BERT-Base} & \textbf{GPT2-Large}   & \textbf{Llama3-8B}   & \textbf{Qwen2.5-14B}      \\\midrule
					FedMIA       &  0.076               & 0.175                   &  0.067                 & 0.090                  \\
					Cosine       &  0.074               & 0.299                   &  0.056                 & 0.073                \\
					Fed-loss     &  0.022               & 0.452                   &  0.738                 & 0.657                \\ 
					Score-Ratio  &  0.037               & 0.593                   &  0.804                 & 0.671                \\ 
					Score-Diff   &  0.037               & 0.598                   &  0.748                 & 0.706                \\ 
					Gradient-Diff&  0.076               & 0.164                   &  0.063                 & 0.065                \\ 
					FTA          &  0.079               & 0.669                   &  0.804                 & 0.806                \\
					Ours      &  0.435               & 2.136                   &  0.434                 & 0.557              \\\bottomrule  
			\end{tabular}}       
		\end{threeparttable}
		\vspace{-0.3cm}
	\end{table}

	\subsection{Impact of Hyperparameters}
	  Here, we aim to evaluate the performance of \texttt{ProjRes} under various hyperparameter settings, \ie the number of clients and the threshold $\tau$, to comprehensively assess its adaptability and robustness across diverse scenarios.

	\para{Impact of the number of clients.} We use the BERT Base model for evaluation on CoLA, and the overall experimental setup was largely consistent with the setup described in \S\ref{main}, with the main difference being the configuration of the number of clients. Specifically, we set the number of clients to $\{5, 10, 20, 50, 100, 200\}$. Due to computational resource limitations, the attack initiation point is adjusted to occur at the 20-th communication round. The results are presented in Fig.~\ref{AUC_client}, which indicate that the performance of \texttt{ProjRes} is independent of the number of clients. This observation aligns with its design principle: \texttt{ProjRes} performs the attack solely based on the gradients of the target client and does not rely on gradient information from other clients.
	
	\para{The impact of threshold $\tau$.} As shown in Eq.~\eqref{eq-27}, \texttt{ProjRes} uses a threshold $\tau$ to distinguish member from non-member samples. Table~\ref{tab:Acc_FPR_t} reports attack accuracy (Acc) and false positive rate (FPR) under different $\tau$ values on CoLA and Yelp, with all other settings as in \S\ref{main}. The results show that an appropriate $\tau$ balances high Acc and low FPR. If $\tau$ is too small, Acc drops due to missed members; if too large, FPR rises as more non-members are misclassified as members. Thus, tuning $\tau$ is critical. In practice, we recommend using $10^{-2}$ or selecting $\tau$ via validation on non-member data.
	
	\begin{table}[!t]
		\renewcommand{\arraystretch}{1.2}
		\setlength{\tabcolsep}{2.5pt}
		\centering
		\caption{AUC scores of ProjRes against Adapter-based Bert-Base under various defenses on CoLA.}
		\label{tab:Acc_FPR_t}
		\begin{threeparttable}
			\begin{tabular}{llcccccc} \hline
				\multirow{2}{*}{\textbf{Dataset}} & \multirow{2}{*}{\textbf{Metric}} & \multicolumn{6}{c}{\textbf{$\tau$}}  \\ \cmidrule(lr){3-8}
				&                 & 5e-2           & 1e-2           & 5e-3         & 1e-3          & 5e-4      & 1e-4 \\ \hline
				\multirow{2}{*}{CoLA}    & Acc             & 92.50\%        & 100\%          & 100\%        & 100\%         & 100\%     & 98.50\%     \\
				& FPR             & 15.00\%        & 0\%            & 0\%          & 0\%           & 0\%       & 0\%   \\ \midrule
				\multirow{2}{*}{Yelp}    & Acc             & 85.00\%        & 92.00\%        & 92.00\%      & 92.00\%       & 92.00\%   & 87.50\%     \\
				& FPR             & 16.00\%        & 0\%            & 0\%          & 0\%           & 0\%       & 0\%      \\\hline
			\end{tabular}
		\end{threeparttable}
		\vspace{-0.2cm}
	\end{table}
	
	\subsection{Attack Performance Evaluation under Active Defense}
	\label{Attack_defense}
    In FedLLMs, computational resources are often limited, making cryptographic methods like homomorphic encryption and secure aggregation impractical due to high overhead. Thus, we focus on two lightweight defenses: differential privacy (DP) \cite{DP} and gradient pruning (GP) \cite{GP}. DP adds noise to gradients, with privacy strength controlled by the noise scale $\sigma$; we evaluate $\sigma \in \{0.01, 0.1, 1, 1.5\}$ and a fixed clipping bound of 1.0. GP reduces leakage by sparsifying gradients, \ie only the largest components are transmitted. We test pruning rates $\beta \in \{70\%, 90\%, 99\%, 99.9\%\}$. Both defenses perturb gradients, potentially harming model utility. Therefore, we assess their impact on test loss and accuracy, enabling a comprehensive analysis of the privacy–utility trade-off and evaluating \texttt{ProjRes}'s effectiveness under practical privacy-preserving conditions.

	\begin{table}[!t]
		\renewcommand{\arraystretch}{1.2}
		\setlength{\tabcolsep}{4pt}
		\centering
		\caption{AUC scores of \texttt{ProjRes} against Adapter-based Bert-Base under various defenses on CoLA.}
		\label{tab:attack_under_DP_GP}
		\begin{threeparttable}
			\begin{tabular}{llcccccc} \hline
				\multirow{2}{*}{\textbf{Defense}} & \multirow{2}{*}{\textbf{Value}} & \multicolumn{5}{c}{\textbf{Batch Size}}  \\ \cmidrule(lr){3-7}
				&                 & 1              & 2              & 4            & 8             & 16 \\ \hline
				\multirow{4}{*}{\textbf{DP ($\sigma$)}}  & $0.01$ & 0.551          & 0.797          & 0.996        & 0.943         & 0.863  \\
				& $0.1$  & 0.582          & 0.568          & 0.534        & 0.532         & 0.524  \\
				& $1$    & 0.490          & 0.520          & 0.501        & 0.522         & 0.507  \\
				& $1.5$  & 0.545          & 0.500          & 0.504        & 0.522         & 0.511  \\ \midrule
				\multirow{4}{*}{\textbf{GP ($\beta$)}}  & $70\%$      & 1.000          & 0.855          & 0.817        & 0.707         & 0.694    \\
				& $90\%$      & 0.954          & 0.788          & 0.714        & 0.589         & 0.550    \\
				& $99\%$      & 0.730          & 0.746          & 0.694        & 0.603         & 0.537    \\
				& $99.9\%$    & 0.609          & 0.779          & 0.664        & 0.580         & 0.513    \\\hline
			\end{tabular}
		\end{threeparttable}
		\vspace{-0.2cm}
	\end{table}
	
Table~\ref{tab:attack_under_DP_GP} summarizes the AUC scores of \texttt{ProjRes} under varying DP and GP settings. The results show that DP only becomes an effective defense when the noise standard deviation exceeds 0.1, but as shown in Table~\ref{tab:loss_accuracy_vertical_style}, this degrades test accuracy below 55.73\%, severely harming model utility. For GP, even at a 99.9\% pruning rate, \texttt{ProjRes} achieves AUC over 0.6 with batch sizes below 4, indicating significant residual leakage. Notably, GP reduces accuracy by only 9.73\%, preserving utility better than DP. However, its protection is highly sensitive to small batch sizes, offering weak guarantees in such cases. These findings highlight a key limitation of current lightweight defenses: they struggle to balance privacy and performance. Results in \S\ref{c3} confirm this trade-off persists across training stages, emphasizing the need for more adaptive, robust privacy mechanisms tailored to FedLLMs.


	\begin{table}[!t]
		\renewcommand{\arraystretch}{1.3}
		\setlength{\tabcolsep}{4pt}
		\centering
		\caption{Loss and accuracy of BERT-Base on CoLA under DP ($\sigma$) and GP ($\beta$). $\uparrow$ and $\downarrow$ denote absolute increase in loss and drop in accuracy relative to baseline.}
		\label{tab:loss_accuracy_vertical_style}
		\begin{tabular}{lccccc}
			\toprule
			\textbf{Defense} & \textbf{Value} & \textbf{Loss} & \textbf{Loss$\uparrow$} & \textbf{Accuracy} & \textbf{Accuracy$\downarrow$} \\
			\midrule
			\multirow{4}{*}{\textbf{DP ($\sigma$)}} 
			& 0.01   & 0.5136 & 0.2796$\uparrow$ & 84.06\% & 7.11\%$\downarrow$ \\
			& 0.1    & 0.7137 & 0.4797$\uparrow$ & 51.03\% & 40.14\%$\downarrow$ \\
			& 1.0    & 0.7129 & 0.4789$\uparrow$ & 50.92\% & 40.25\%$\downarrow$ \\
			& 1.5    & 0.6860 & 0.4520$\uparrow$ & 55.73\% & 34.44\%$\downarrow$ \\
			\midrule
			\multirow{4}{*}{\textbf{GP ($\beta$)}}
			& 70\%   & 0.2812 & 0.0017$\uparrow$ & 90.71\% & 0.46\%$\downarrow$ \\
			& 90\%   & 0.2812 & 0.0472$\uparrow$ & 89.33\% & 1.84\%$\downarrow$ \\
			& 99\%   & 0.3654 & 0.1314$\uparrow$ & 87.61\% & 3.56\%$\downarrow$ \\
			& 99.9\% & 0.4913 & 0.2573$\uparrow$ & 81.42\% & 9.75\%$\downarrow$ \\
			\bottomrule
		\end{tabular}
		\vspace{-0.3cm}
	\end{table}

	\section{Related Work}
	
	\para{Federated Large Language Models.}
	The growing privacy sensitivity of user data and the limited computational capacity of edge users have created a strong demand for distributed fine-tuning of LLMs~\cite{adapter_efficent}. In this context, FL has naturally converged with LLMs, giving rise to FedLLMs, which is a rapidly emerging research focus across academia~\cite{openfedllm, FedLLM_future, FedLLM_FL2} and industry~\cite{federatedscope}. Early studies explored direct integration of LLMs into FL for full-parameter end-to-end fine-tuning \cite{FedBERT, legal1}, but the enormous parameter scales of modern models make such approaches impractical for resource-constrained clients. To overcome this, PEFT techniques, such as Adapters \cite{adapter_structure, adapter_efficent, adapter_Heterogeneous, adapter_Dual-Personalizing, adapter}, LoRA \cite{LoRA, LoRA_FT, LoRA_Personalized, LoRA_Heterogeneous}, and partial parameter tuning \cite{FedRDMA, lin_2024, 10097124}, have been adopted in FedLLMs. These methods train only a small subset of parameters, drastically reducing computation and communication costs while preserving accuracy. 

	\para{MIAs in FL.} MIAs~\cite{MIA_Survey} in FL aim to determine whether a specific sample is part of a client's training set, posing serious privacy risks. Attackers may exploit shared gradients, historical model updates, intermediate computations, or final outputs. Broadly, MIAs fall into two categories: trend-based and update-based. Trend-based attacks \cite{trend_based_MIA1, trend_based_MIA2, learn-based_MIA1, learn-based_MIA2, learn-based_MIA4} track model updates over time, extracting membership-related metrics and comparing their distributions across samples. They often rely on shadow models trained with auxiliary data, but their need for repeated feature collection and retraining leads to high computational costs. To mitigate this, \cite{FTA} compares the rate of change of metrics between member and non-member samples, reducing attack overhead. Update-based attacks, by contrast, use information from a single training round, inferring membership from differences in loss or gradient properties. For example, members tend to cause larger loss variations \cite{ICLR_MIA, ref_MIA2, FedMIA} or exhibit distinct gradient behaviors such as non-orthogonality \cite{ICLR_MIA}. While more efficient, these methods depend heavily on specific model characteristics.
	
	\section{Discussion and Future Work}
	
	This work presents \texttt{ProjRes}, a novel projection–residual–based MIA tailored for FedLLMs. While \texttt{ProjRes} effectively exploits the geometric structure of client-updated gradient subspaces, several open challenges and promising research directions remain.
	
	\para{Limitations.}
	\texttt{ProjRes} reveals privacy leakage risks in the training process of FedLLMs but has two main limitations. First, it does not propose new defense mechanisms that are both effective and efficient. As shown in \S\ref{Attack_defense}, while DP and GP are lightweight privacy-preserving methods, they exhibit a sharp trade-off between protection strength and model utility, stronger defenses lead to significant accuracy degradation. This underscores the need for new privacy mechanisms that better balance security and performance. Second, \texttt{ProjRes} operates as a data-level MIA, identifying membership only when a sample exactly matches one in the training set. Although consistent with the classical MIA definition, this overlooks the semantic generalization ability of LLMs. Developing semantic-level MIAs that capture broader privacy risks in FedLLMs is an important direction for future research.

	\para{Adaptive and Defense-Aware Attacks.}
	Another direction involves designing adaptive \texttt{ProjRes} variants that account for defense mechanisms such as gradient clipping, differential privacy noise, or adversarial regularization. Integrating statistical priors or Bayesian inference into residual modeling could enhance robustness under noise perturbation.
	
	\para{Towards Comprehensive Privacy Auditing.}
	Finally, combining \texttt{ProjRes} with complementary modalities, such as training dynamics, loss curvature, or representational entropy, may lead to more comprehensive privacy auditing tools for large-scale federated systems. Such hybrid frameworks could serve as diagnostic instruments to quantify privacy risks and guide the development of stronger, defense-aware training protocols for future FedLLMs.

	\section{Conclusion}
	
	In this work, we propose \texttt{ProjRes}, a gradient representation residual-based MIA specifically designed for FedLLMs. Our findings reveal that the gradients generated during training from commonly employed fully connected components, \eg Adapters, LoRA modules, query/key/value projection matrices, and MLPs, can result in substantial privacy leakage. Through theoretical analysis, we further demonstrate that the severity of this leakage is strongly correlated with the number of neurons in these fully connected structures: larger neuron counts amplify the leakage signal. Due to the deterministic algebraic relationship between sample features and model gradients, \texttt{ProjRes} achieves stable and highly effective attack performance across various FedLLMs.
	

	
	\clearpage
	\bibliographystyle{IEEEtran}
	\bibliography{sample}
	%
		
		
	
	\appendices
	\begin{figure*}[!t]
		\centering
		\includegraphics[width=0.9\textwidth]{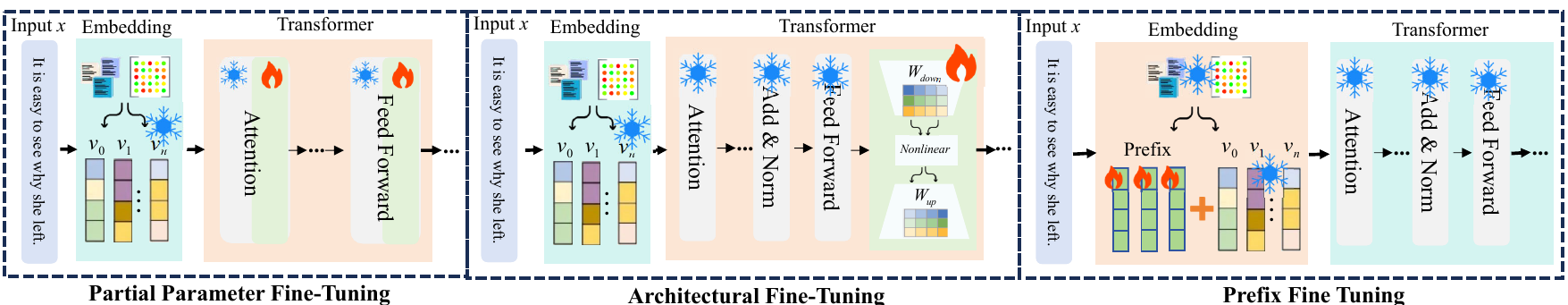} %
		\caption{An overview of three FedLLM fine-tuning strategies.}
		\label{P-MIA-1}
		\vspace{-0.5cm}
	\end{figure*}

	\section*{Ethics Considerations}
    None.
	
	\section{FedLLMs Fine-tuning Strategy}\label{finetuning}
	As shown in Fig. \ref{P-MIA-1}, FedLLMs has three commonly used fine-tuning strategies, which are described in detail below:

	\para{Partial Parameter Fine-Tuning.} In partial fine-tuning, only selected layers~\cite{FedRDMA} or modules of the LLM are updated~\cite{lin_2024, 10097124} (\eg self-attention modules, layer norms, or feed-forward MLP). Let $L_s$ denote the subset of layers chosen for training. Then:
	\begin{equation}
		\theta_{\text{trainable}} = {\theta_i ,|, i \in L_s}, \quad \theta_{\text{frozen}} = {\theta_j ,|, j \notin L_s}.	
	\end{equation}
	Each client $k$ performs local updates via stochastic gradient descent (SGD):
	\begin{equation}
		\theta_{\text{trainable}}^{(t+1, k)} = \theta_{\text{trainable}}^{(t)} - \eta \nabla_{\theta_{\text{trainable}}} \mathcal{L}_k(f_{\theta^{(t)}}),	
	\end{equation}
	where $\eta$ is the learning rate. Precisely because the LLM obtained through joint fine-tuning with such methods is structurally identical to the original model and differs only in a small subset of parameters, it is referred to as partial parameter fine-tuning. This approach reduces computation while retaining adaptability in critical layers~\cite{bitfit_few-shot, bitfit}.

	\para{Architectural Fine-Tuning.} Architectural fine-tuning introduces lightweight trainable submodules into the frozen backbone of the LLM. For adapter tuning~\cite{adapter_structure, adapter_efficent, adapter_Heterogeneous, adapter_Dual-Personalizing, adapter}, small bottleneck layers $h_\phi(\cdot)$ with parameters $\phi$ are inserted into each transformer block:
	\begin{equation}
		h_\phi(\bm x) = \bm W_{\text{up}} , \sigma(\bm W_{\text{down}} \bm x),	
	\end{equation}
	where $\bm W_{\text{down}} \in \mathbb{R}^{n \times p}$, $\bm W_{\text{up}} \in \mathbb{R}^{p \times n}$, and $p \ll n$. The overall model becomes:
	\begin{equation}
		f_{\phi}(\bm x) = h_\phi(\bm x)+\bm x,
	\end{equation}
	where only $\phi$ is updated locally.
	
	Alternatively, Low-Rank Adaptation (LoRA)~\cite{LoRA, LoRA_FT, LoRA_Personalized, LoRA_Heterogeneous} fine-tunes low-rank matrices $ \bm A \in \mathbb{R}^{n \times p} $, $ \bm B \in \mathbb{R}^{p \times n}$ that modify existing weight matrices $\bm W$:
	\begin{equation}
		\bm W' = \bm W + \Delta \bm W = \bm W + \bm B \bm A, 
	\end{equation}
	Thus, the gradient updates are restricted to $\{ \bm A, \bm B\}$:
	\begin{equation}
		\bm A^{(t+1)} = \bm A^{(t)} - \eta \nabla_{\bm A} \mathcal{L}_k, 
		\bm B^{(t+1)} = \bm B^{(t)} - \eta \nabla_{\bm B} \mathcal{L}_k.	
	\end{equation}
	
	\para{Prefix Fine-Tuning.} In prefix fine-tuning, clients learn a small set of trainable prefix embeddings that condition the attention mechanism, without modifying the original LLM parameters~\cite{Parameter-Efficient_Prompt_Tuning, P-tuningV2}. Let $ \bm P \in \mathbb{R}^{p \times d}$ denote the prefix parameters (with $p \ll L$, where $L$ is the sequence length), which are concatenated to the input sequence during self-attention:
	\begin{equation}
		\text{Attn}(\bm Q, \bm K, \bm V; \bm P) = \text{softmax}\left( \frac{[\bm Q; \bm P_{\bm Q}][\bm K; \bm P_{\bm K}]^T}{\sqrt{d_k}} \right) [\bm V; \bm P_{\bm V}],	
	\end{equation}
	where $\bm P_{\bm Q}, \bm P_{\bm K}, \bm P_{\bm V}$ are learned prefix keys and values. The optimization objective per client becomes:
	\begin{equation}
		\min_{\bm P_k} \mathcal{L}_k(f_{\theta_{\text{frozen}}, \bm P_k}),	
	\end{equation}
	and after training, only the prefixes $\bm P_k$ are sent back to the server. The core idea of this approach is to optimize the model's input representations, thereby steering the pretrained model to better adapt to downstream tasks and fully leverage its pre-acquired knowledge and capabilities. However, this design also inherently limits their adaptability to diverse downstream tasks to some extent.
	
	\section{Theoretical Analysis}\label{sec-3-5}
The core of \texttt{ProjRes} lies in extracting the semantic component $ f(\bm{x})^{\text{rec}} $ from model gradients $\nabla \bm{W}$ and comparing it with the true semantics $ f(\bm{x}) $ to perform membership inference. This process relies on two key conditions: (i) the gradients must capture rich semantic information from member samples, governed by the architectural dimensions $n$ and $m$; and (ii) the semantic signal must be effectively disentangled from the gradients, which depends on the number $p$ of input embedding vectors. Given a fixed fully connected architecture (fixed $n$, $m$), this section analyzes the attack capability boundary of \texttt{ProjRes}—specifically, the maximum $p$ for which semantic recovery remains effective. We further investigate whether increasing $n$ and $m$ via upsampling mechanisms, common in LLMs, can expand this boundary, potentially enhancing both semantic expressiveness and disentanglement for stronger attacks.
	
	\subsection{Attack Capability Boundary of \texttt{ProjRes}}
	\label{Basic_Capability}
	
	In FedLLMs, the architecture of the trainable module is fixed during server initialization, whereas the number of hidden embedding vectors in client's training data is changing. Consequently, this section analyzes the maximum feasible value of \( p \) under given values of \( m \) and \( n \).
	
	\begin{theorem}
		Suppose \texttt{ProjRes} is effective, the maximum number $p_\text{max}$ of hidden embedding vectors fed into the fully connected layer and the layer’s architectural parameters \( n \) and \( m \) (where \( n \) and \( m \) denote the input and output neuron counts, respectively) must satisfy the following constraint:
		\begin{equation}
			p_\text{max} \leq \min(n-1, m)	
		\end{equation}
		This constraint ensures that the subspace spanned by the gradients of the fully connected layer remains full-rank and uniquely determined by the member samples. If $p_\text{max}$ exceeds either \( n \) or \( m \), the embedding vectors become linearly dependent in the gradient space, resulting in projection residuals that no longer preserve meaningful membership distinctions.
	\end{theorem}

	\begin{proof}
		According to Eq. \eqref{eq:p_max_condition}, \( p_{\max} \) satisfies the following:
		\begin{equation}
			p_{\max} < n \quad \text{and} \quad p_{\max} \leq m .
			\label{eq:p_max_condition}
		\end{equation}
		Under this condition, Eqs. \eqref{eq:Span_x_w} and \eqref{eq:x_rec} hold naturally, and thus \texttt{ProjRes} is effective. Conversely, suppose there exists some \( p' > p_{\max} \) for which \texttt{ProjRes} remains effective. Depending on the relative magnitudes of \( m \) and \( n \), we consider the following two cases: (1) \( p' > p_{\max} \) and \( m \geq n \); (2) \( p' > p_{\max} \) and \( m < n \).
		
		\para{Case 1.} If \( p' > p_{\max} \) and \( m \geq n \), Eq. \eqref{eq:rank_w} and \eqref{eq:rank_x} can be rewritten as follows
		\begin{equation}
			\operatorname{rank}(\nabla \bm{W}_{n\times m})=\operatorname{min}(n, p', m)=n,
			\label{eq:rank_w_new}
		\end{equation}
		\begin{equation}
			\operatorname{rank}(\bm{X})=\operatorname{min}(n, p')=n.
			\label{eq:rank_x_new}
		\end{equation}
		In this case, although \(\operatorname{Span}(\nabla \bm{W}) = \operatorname{Span}(\bm{X})\) still holds, each vector \(\bm{x}_i \in \bm{X}\) has dimension \(n\), making \(\operatorname{Span}(\bm{X})\) a full-rank subspace. Consequently, its representational capacity is no longer confined to a low-dimensional subspace spanned by a limited basis, which renders \texttt{ProjRes} ineffective.
		
		\para{Case 2.} If \( p' > p_{\max} \) and \( m < n \), Eq. \eqref{eq:rank_w} and \eqref{eq:rank_x} can be rewritten as follows
		\begin{equation}
			\operatorname{rank}(\nabla \bm{W}_{n\times m})=\operatorname{min}(n, p', m)=m,
			\label{eq:rank_w_new2}
		\end{equation}
		\begin{equation}
			\operatorname{rank}(\bm{X})=\operatorname{min}(n, p') \neq m.
			\label{eq:rank_x_new2}
		\end{equation}
		In this case, \(\operatorname{Span}(\nabla \bm{W}_{nm}) \neq \operatorname{Span}(\bm{X})\), as the limited number \(m\) of neurons is insufficient to fully capture the information contained in up to \(p'\) hidden embedding vectors, rendering \texttt{ProjRes} ineffective. Combining both cases, no \(p' > p_{\max}\) exists for which \texttt{ProjRes} remains effective. Therefore, the maximum feasible value of \(p\) is precisely \(p_{\max}\).
		
	\end{proof}
	
	\subsection{Impact of Upsampling}
	In LLMs, dense layers commonly expand their dimensionality through upsampling. For instance, in BERT-Base, the dense layer in the output module has an input dimension of 3072 after upsampling, that is four times the hidden embedding dimension of 768. Similarly, in GPT-Large, the \texttt{c\_fc} layer within the MLP block expands to an input dimension of 5120, which is four times its hidden embedding dimension of 1280. Given that such dimensional expansion may enlarge the attack surface, this section investigates whether the increased dimensions of the input and output layers, induced by upsampling, enhance the attack capability of \texttt{ProjRes}. Let the input to this module before up-projection remain denoted as \( \bm{X} \), and let the output of the up-projection layer be \( \bm{X}^{\text{up}} \). Then, we have:
	\begin{equation}
		\bm{X}_{\text{up}}=\bm{W}_{\text{up}}\bm{X},
		\label{eq:up-projection}
	\end{equation}
	where $ \bm{W}_{\text{up}} $ denotes the weight matrix of the up-projection layer, with dimensions $ n_{\text{up}} \times n $, and \( n_{\text{up}} > n \). The resulting up-projected output \( \bm{X}_{\text{up}} \) thus has dimensions \( n_{\text{up}} \times p \). The ranks of \( \bm{X}_{\text{up}} \), \( \bm{W}_{\text{up}} \), and \( \bm{X} \) satisfy the following relationship:
	\begin{equation}
		\begin{split}
			\operatorname{rank}(\bm{X}_{\text{up}})= \operatorname{rank}(\bm{W}_{\text{up}} \bm{X})
			&= \min(n_{\text{up}}, p, n) \\
			&= \min(p, n) \\
			&= \operatorname{rank}(\bm{X}) \\
		\end{split}.
		\label{eq:rank_n_up}
	\end{equation}
	As shown in Eq. \eqref{eq:rank_n_up}, the effective information contained in the up-projected hidden embeddings $ \bm{X}_{\text{up}} $ is still determined by the original hidden embeddings $ \bm{X} $. Consequently, under the \texttt{ProjRes} attack, the amount of information leaked through model gradients is fundamentally governed solely by the dimensionality of the original hidden embeddings.
	
	\section{Full Description of \texttt{ProjRes} Algorithm}
	\label{full_description}
	
	\begin{algorithm}[!t]
		\caption{Projection Residual-based MIA}
		\label{alg:P-MIA}
		\begin{algorithmic}[1] 
			\Require \\
			Global model $f_\theta$; \\
			Client shared gradient $\nabla \theta$; \\
			Threshold $\tau$; \\
			Target sample $\bm x$.
			\Ensure
			Membership inference.
			
			\State Compute the hidden embedding vector $f(\bm x)$.
			\State Construct the low-rank linear subspace: 
			
			$$\mathcal{S} = \operatorname{Span}(\nabla \bm{W})$$
			
			\State Reconstruct the embedding vector:
			
			$$f(\bm x)^{\text{rec}} = \Pi_{\mathcal{S}}\big( f(\bm x) \big)$$
			
			\State Compute the residual $r$:
			\[
			r = \big\| f(\bm x) - f(\bm x)^{\text{rec}} \big\|_1
			\]
			\State Conduct membership inference:
			$$CF(\bm x,f(\bm x)^{\mathrm{rec}}) =
			\begin{cases}
				1, & \text{if }  r < \tau, \\
				0, & \text{otherwise},
			\end{cases}	.$$
			\Return Membership
		\end{algorithmic}
	\end{algorithm}
	
	Although we have given the necessary steps for the \texttt{ProjRes} implementation in Algorithm~\ref{alg:P-MIA}, the steps for reconstructing the embedding vector $f(\bm x)^{\text{rec}}$ still need further refinement. Therefore, we provide a detailed supplement here. First, reconstructing the embedding vector \(f(\bm x)\) in the vector space $\mathcal{S}$ is essentially equivalent to solving for a set of coefficients \(\bm{\alpha}\), which requires satisfying the following conditions: $f(\bm x) = \mathcal{S} \cdot \bm{\alpha}$.	Since the linear space \(\mathcal{S} \in \mathbb{R}^{m\times n}\) has rank properties, these properties determine the existence and uniqueness of the coefficient vector \(\boldsymbol{\alpha}\). Specifically, we have the following setps:
	
	\textit{Step 1: QR Decomposition.} We conduct QR decomposition for $\mathcal{S}$, thus, we have:$\mathcal{S} = \bm{Q}\bm{R}$, where \( \bm{Q} \in \mathbb{R}^{d \times k} \) is a column-orthogonal matrix (\( \bm{Q}^\top \bm{Q} = \bm{I}_k \)), \( \bm{Q}\bm{R} \in \mathbb{R}^{k \times m} \) is an upper triangular matrix.
	
	\textit{Step 2: Projection and Triangular Solution.} Substitute the original equation into the QR decomposition: $\bm Q\bm{R}\boldsymbol{\alpha} = f(\bm x).$ Left-multiply by $ \bm Q^\top $ (using $ \bm Q^\top \bm Q = \bm{I} $), thus, we have: $\bm{R}\boldsymbol{\alpha} = 	\bm Q^\top f(\bm x).$ Since $\bm{R}$ is an upper triangular matrix, $ f(\bm x) $ can be efficiently solved using back substitution, \ie
	\begin{equation}
		\boldsymbol{\alpha} = \bm{R}^{-1}(\bm Q^\top f(\bm x))
	\end{equation}
	
	\textit{Step 3: Reconstruct the Target Vector.} After obtaining the coefficient vector $ \bm{\alpha} $, the reconstructed vector is given by:
	\begin{equation}
		f(\bm x)^{\text{rec}} = \mathcal{S} \cdot \boldsymbol{\alpha} = \bm{Q}\bm{R} \bm{R}^{-1}(\bm{Q}^\top f(\bm x))
	\end{equation}
	where $ f(\bm x)^{\text{rec}} $ is is the orthogonal projection of $ f(\bm x) $ onto this subspace $ \mathcal{S} $. 
	
	\para{Remark.} We emphasize two key points. First, the use of QR decomposition is merely a practical choice for computing the projection of $f(\bm x)$ onto the linear subspace $\mathcal{S}$; fundamentally, this remains a well-studied algebraic problem with multiple possible solutions. Second, in this work, we adopt QR decomposition to obtain the reconstructed vector, where the primary computational cost arises from the decomposition step itself. While optimizing this step could further enhance the overall efficiency of \texttt{ProjRes}, computational efficiency is not the main focus of this paper, and we therefore leave such optimizations for future work.

	\section{Supplementary Experimental Results}\label{supp}
In the main experiments, we evaluated \texttt{ProjRes} on the CoLA dataset, analyzing its performance across training epochs (\S\ref{sub_AUC_epoch}) and different trainable module placements (\S\ref{layer_auc}). While these results demonstrate \texttt{ProjRes}'s robustness, CoLA's simplicity leads to consistently high attack performance, limiting the ability to observe meaningful variations. To better capture performance dynamics, we conduct additional experiments on the Yelp dataset, which exhibits more pronounced fluctuations in \texttt{ProjRes}'s effectiveness (as shown in Fig.~\ref{fig:roc_comparison}) and thus serves as a more revealing benchmark. We further provide a supplementary evaluation of DP and GP defense performance across training stages.
	
	\para{Impact of Training Stages.} 
	\label{c1}
Fig.~\ref{AUC_epoch} shows the performance of \texttt{ProjRes} against an Adapter-based BERT-Base on Yelp across training epochs, revealing a gradual decline starting around epoch 20. However, this degradation may stem not only from overfitting but also from evolving neuron activation patterns within the Adapter. To isolate this effect, we present results on LoRA-based FedLLMs in Fig.~\ref{AUC_epoch_yelp_Lora}, which support our hypothesis: the attack performance of \texttt{ProjRes} is closely tied to the representational capacity of the gradient-based space $\mathcal{S}$, relative to the number of input hidden embeddings. When the number of active neurons significantly exceeds the embedding dimension, $\mathcal{S}$ remains rich and stable, yielding consistent attack performance. In contrast, when capacity is constrained, prolonged training increases activation sparsity—leading to neuron deactivation and a mild decline in \texttt{ProjRes} effectiveness. Thus, the interplay between model activation states and input scale critically influences the attack dynamics.
	
	\begin{figure}[!t]
		\centering
		\includegraphics[width=0.3\textwidth]{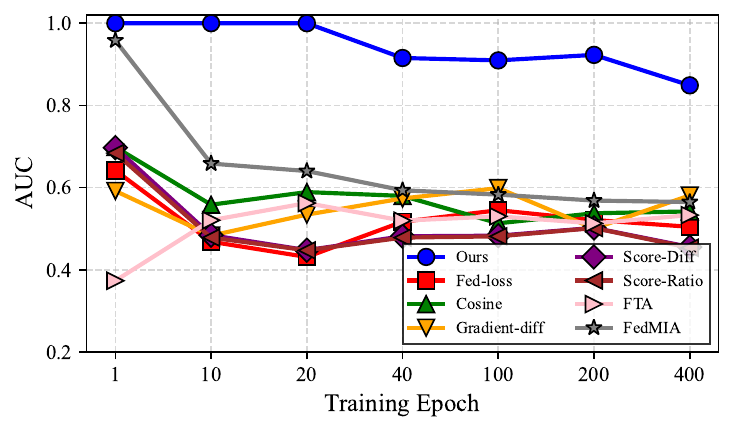} %
		\caption{AUC scores of MIAs agaist Adapter-based BERT-Base on Yelp across different training epochs.}
		\label{AUC_epoch_yelp}
		\vspace{-0.4cm}
	\end{figure} 
	
	\begin{figure}[!t]
		\centering
		\includegraphics[width=0.3\textwidth]{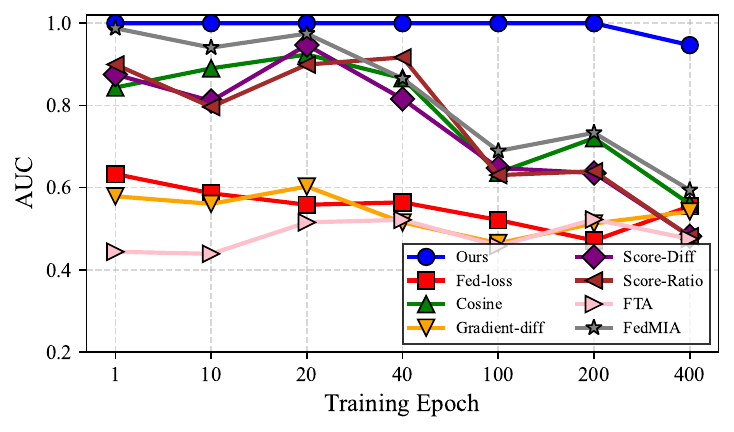} %
		\caption{AUC scores of MIAs agaist LoRA-based BERT-Base on Yelp across different training epochs.}
		\label{AUC_epoch_yelp_Lora}
			\vspace{-0.4cm}
	\end{figure} 
	
	
	\para{Impact of Trainable Layer Positions} \label{c2}
	Fig. \ref{AUC_layer_yelp} presents the evaluation results of \texttt{ProjRes} against BERT-Base with a single Adapter in different position on Yelp, it showes that \texttt{ProjRes} achieves an AUC of 1.0 for attacks targeting adapters inserted in all Transformer layers except the 11-th, where a slight performance drop is observed. This finding is largely consistent with the main conclusion presented in the main text and further corroborates that \texttt{ProjRes} is insensitive to the placement of trainable modules.
	
	\begin{figure}[!t]
		\centering
		\includegraphics[width=0.3\textwidth]{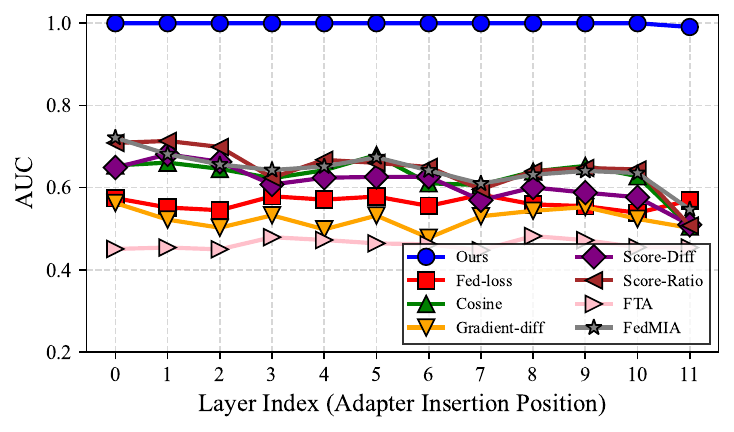} %
		\caption{The AUC scores of \texttt{ProjRes} against BERT-Base with a single adapter inserted at different Transformer layers.}
		\label{AUC_layer_yelp}
			\vspace{-0.2cm}
	\end{figure}

		\begin{table}[!t]
		\renewcommand{\arraystretch}{1.2}
		\setlength{\tabcolsep}{4pt}
		\centering
		\caption{AUC scores of \texttt{ProjRes} against Adapter-based BERT-Base under various defenses on CoLA after 1 epoch.}
		\begin{threeparttable}
			\begin{tabular}{llcccccc} \hline
				\multirow{2}{*}{\textbf{Defense}} & \multirow{2}{*}{\textbf{Value}} & \multicolumn{5}{c}{\textbf{Batch Size}}  \\ \cmidrule(lr){3-7}
				&                 & 1              & 2              & 4            & 8             & 16 \\ \hline
				\multirow{4}{*}{\textbf{DP ($\sigma$)}}  & $ 0.01$ & 1.000          & 1.000          & 1.000        & 0.999         & 0.848  \\
				& $0.1$  & 1.000          & 0.911          & 0.663        & 0.565         & 0.537  \\
				& $ 1$    & 0.495          & 0.502          & 0.514        & 0.491         & 0.510  \\
				& $1.5$  & 0.520          & 0.519          & 0.499        & 0.507         & 0.517  \\ \midrule
				\multirow{4}{*}{\textbf{GP ($\beta$)}}  & $70\%$      & 1.000          & 1.000          & 0.995        & 0.975         & 0.772    \\
				& $90\%$      & 1.000          & 0.922          & 0.766        & 0.629         & 0.544    \\
				& $99\%$      & 0.999          & 0.740          & 0.630        & 0.531         & 0.508    \\
				& $99.9\%$    & 0.515          & 0.542          & 0.530        & 0.552         & 0.503    \\\hline
			\end{tabular}
		\end{threeparttable}
		\label{DP_GP_epoch_1}
		\vspace{-0.2cm}
	\end{table}
	
	\begin{table}[!t]
		\renewcommand{\arraystretch}{1.2}
		\setlength{\tabcolsep}{4pt}
		\centering
		\caption{AUC scores of \texttt{ProjRes} against Adapter-based BERT-Base under various defenses on CoLA after 20 epoches.}
		\begin{threeparttable}
			\begin{tabular}{llcccccc} \hline
				\multirow{2}{*}{\textbf{Defense}} & \multirow{2}{*}{\textbf{Value}} & \multicolumn{5}{c}{\textbf{Batch Size}}  \\ \cmidrule(lr){3-7}
				&                 & 1              & 2              & 4            & 8             & 16 \\ \hline
				\multirow{4}{*}{\textbf{DP ($\sigma$)}}  & $ 0.01$ & 0.760          & 0.633          & 0.972        & 0.973         & 0.828  \\
				& $ 0.1$  & 0.550          & 0.591          & 0.565        & 0.574         & 0.490  \\
				& $ 1$    & 0.508          & 0.502          & 0.467        & 0.490         & 0.480  \\
				& $1.5$  & 0.556          & 0.512          & 0.510        & 0.476         & 0.504  \\ \midrule
				\multirow{4}{*}{\textbf{GP ($\beta$)}}  & $70\%$      & 1.000          & 0.859          & 0.819        & 0.814         & 0.754    \\
				& $90\%$      & 0.998          & 0.786          & 0.697        & 0.616         & 0.563    \\
				& $99\%$      & 0.941          & 0.782          & 0.675        & 0.627         & 0.599    \\
				& $ 99.9\%$    & 0.798          & 0.761          & 0.673        & 0.612         & 0.528    \\\hline
			\end{tabular}
		\end{threeparttable}
		\label{DP_GP_epoch_20}
		\vspace{-0.2cm}
	\end{table}

	\para{Attack Performance Evaluation under Active Defense.} \label{c3} The supplementary experiments in this section aim to examine the defensive effectiveness of DP and GP mechanisms of equal strength against \texttt{ProjRes} attacks at different stages of training. A comparison of the results in Tables~\ref{DP_GP_epoch_1}, \ref{DP_GP_epoch_20}, and \ref{tab:attack_under_DP_GP} (in \S\ref{Attack_defense}) reveals that DP and GP with the same noise magnitude offer weaker defense in the early training stages compared to later ones. For example, when the noise scale is set to \(\sigma = 0.01\), \texttt{ProjRes} achieves significantly higher attack performance after epoch~1 than after epoch~10.


\end{document}